\documentclass[letterpaper,11pt]{article}
\usepackage[margin=1in]{geometry}
\usepackage[bookmarks, colorlinks=true, plainpages = false, citecolor = blue,linkcolor=red,urlcolor = blue, filecolor = blue]{hyperref}
\usepackage{url}
\usepackage{amsmath,amsfonts,amsthm,amssymb,bm,verbatim,dsfont,mathtools,amscd,bm,epsfig,epsf,bbm,mathabx}
\usepackage{algorithm,algorithmic,color,graphicx,appendix}
\usepackage{}
\usepackage{cases,subfigure,float,etoolbox,tikz}
\usepackage{xr,xspace}
\usepackage{todonotes,paralist,caption,prettyref}
\usepackage[]{natbib}

\newrefformat{alg}{Algorithm~\ref{#1}}
\theoremstyle{plain}
\newtheorem{theorem}{Theorem}
\newtheorem{lemma}{Lemma}

\theoremstyle{definition}

\newtheorem*{remark*}{Remark}

\numberwithin{equation}{section}
\usepackage{color}

\usepackage{xspace,prettyref}

\newcommand{\expect}[1]{\mathbb{E}\left[ #1 \right]}

\newcommand{\var}{\mathsf{var}}

\newrefformat{eq}{(\ref{#1})}
\newrefformat{chap}{Chapter~\ref{#1}}
\newrefformat{sec}{Section~\ref{#1}}
\newrefformat{algo}{Algorithm~\ref{#1}}
\newrefformat{fig}{Fig.~\ref{#1}}
\newrefformat{tab}{Table~\ref{#1}}
\newrefformat{rmk}{Remark~\ref{#1}}
\newrefformat{clm}{Claim~\ref{#1}}
\newrefformat{def}{Definition~\ref{#1}}
\newrefformat{cor}{Corollary~\ref{#1}}
\newrefformat{lmm}{Lemma~\ref{#1}}
\newrefformat{prop}{Proposition~\ref{#1}}
\newrefformat{app}{Appendix~\ref{#1}}
\newrefformat{hyp}{Hypothesis~\ref{#1}}
\newrefformat{thm}{Theorem~\ref{#1}}

\newcommand{\norm}[1]{\left\|{#1} \right\|}

\newcommand{\diag}{{\rm diag}}

\definecolor{darkred}{RGB}{100,0,0}
\definecolor{darkgreen}{RGB}{0,100,0}
\definecolor{darkblue}{RGB}{0,0,150}
\definecolor{red}{RGB}{255,0,0}

\newcommand{\vct}[1]{\bm{#1}}
\newcommand{\mtx}[1]{\bm{#1}}

\newcommand{\transp}{T}

\newcommand{\rank}{\operatorname{rank}}

\renewcommand{\hat}{\widehat}
\renewcommand{\tilde}{\widetilde}

\newcommand{\silent}[1]{}

\definecolor{xl}{RGB}{200,50,120}
\definecolor{yc}{RGB}{50,150,50}
\definecolor{jx}{RGB}{50,50,200}

\renewcommand{\v}[1]{ {#1}_1, {#1}_2, \cdots, {#1}_n }

\def \figpath {}

\begin{document}

\title{Clustering Degree-Corrected Stochastic Block Model with Outliers}

\date{}
\author{Xin Qian\thanks{Northwestern University, Industrial Engineering and Management Sciences, Ithaca, NY, 14850, USA, email: {\tt XinQian2017@u.northwestern.edu}.}
\ \ \ \ \ \ \ \
Yudong Chen\thanks{Cornell University, School of Operations Research and Information Engineering, Ithaca, NY, 14850, USA, email: {\tt yudong.chen@cornell.edu}.}   \ \ \ \ \ \ \ \
Andreea Minca\thanks{Cornell University, School of Operations Research and Information Engineering, Ithaca, NY, 14850, USA, email: {\tt acm299@cornell.edu}.}
}

\maketitle

\begin{abstract}
For the degree corrected stochastic block model in the presence of arbitrary or even adversarial outliers, we develop a convex-optimization-based clustering algorithm that includes a penalization term depending on the positive deviation of a node from the expected number of edges to other inliers. We prove that under mild conditions, this method achieves exact recovery of the underlying clusters.    Our synthetic experiments show that our algorithm performs well on heterogeneous networks, and in particular those with Pareto degree distributions, for which outliers have a broad range of possible degrees that may enhance their adversarial power. We also demonstrate that our method allows for recovery with significantly lower error rates compared to existing algorithms. 
\end{abstract}


\section{Introduction}
\label{sec:intro}

Clustering nodes in a complex network is one of the major challenges in network science. Various aspects of this problem have been studied by researchers across different fields including computer science, statistics, operation research,  probability and physics; for a partial list of work in this line, see
\cite{ABFX2008,BC2009,bordenave,chen2012sparseclustering,CNM2004,Cordon01,DHM2004,DEFI06,FB2007,fortunato2010community,Vershynin14,LR2013,massoulie,NG2004,NS2001,RCY2011,Shamir11budget,yi2012semi,ZLZ14}.

A variety of clustering algorithms have been developed,  such as modularity maximization~\citep{Newman2006modularity,KN2011}, graph-cut based methods~\citep{Cordon01,bollobas2004maxcut},
model/likelihood-based methods~\citep{ZLZ2012,LLV14}, spectral clustering~\citep{McSherry2001,shi2000normalizedcut,NgJordanWeiss}, hierarchical clustering
methods~\citep{balcangupta}, and more recently algorithms for dynamic networks \citep{hajek2018community}.

A main goal and driving force for the development of new algorithms is to tackle features of real world datasets, among those size and heterogeneity of the networks. In many cases, promising algorithms remain heuristic and  are not yet amenable to rigorous performance analysis.  On the theory side, obtaining provable performance guarantees is hindered by the fact that each  realistic feature added to the network model significantly increases the complexity of the analysis.

A recent line of research on clustering makes use of convex optimization to achieve both computational efficiency and statistical quality guarantees ~\citep{Jalali2011clustering,chen2012sparseclustering,ames2011plantedclique,Cai2014robust,Vershynin14}. By relaxing the original combinatorial problem into a semidefinite program (SDP), tractable clustering algorithms are developed which, under very general statistical settings, provably produce a high quality clustering. With a few exceptions such as~\cite{DHM2004,KN2011,CLX2017}, most previous  work in this area considers homogeneous networks, in which different nodes exhibit similar statistical properties.

In this paper, we focus on designing clustering algorithms applicable to heterogeneous networks with outliers, and on providing theoretical guarantees for them.  In particular, we would like to simultaneously capture the following three key features common in real world networks:
\begin{itemize}
    \item \textbf{Clustering/community structure:} Nodes may belong to several different groups, where nodes in the same group are more likely to connect to each other than those in different groups.
    \item \textbf{Heterogeneous degrees:} It is well-documented that real-world networks have heterogeneous node degrees. In particular, the degrees of nodes, even those in the same cluster, may exhibit heavy-tailed distribution~\citep{newmanbook}.
    \item \textbf{Arbitrary outliers:} There may exist a set of nodes that do not belong to any clusters and have arbitrary, even adversarial connection patterns.
\end{itemize}
Many networks have these properties. For example, the political blogs networks~\citep{Adamic2005} contain blogs that are mostly democratic or republican-oriented, some of whom have significantly more followers than others, but  there are blogs that are associated with neither political groups.
 Another important example is given by financial networks. These  consist of thousands of nodes at most but face a lot of heterogeneity. More importantly, the number of clusters is known and small. Clusters represent classical investment strategies, while some of the firms are multi-strategy and can be thought of as outliers ~\citep{guo2016}. Arbitrary outliers could also correspond to non-bank firms connected in financial derivatives markets \citep{peltonen2014}. Last and probably most important, for genomic networks, there are a variety of  co-expression networks that suffer from the presence of outliers. For example, cancer-type-specific co-expression networks are medium sized network of 10-20 thousand genes and network analysis and clustering can be useful to identify prognostic genes for some types of cancer, \cite{yang2014}.

Note that the outliers may be different in nature among themselves, so one cannot simply treat them as an additional cluster and apply a standard clustering algorithm. Indeed, many existing methods, such as spectral clustering, are known to perform poorly in the presence of outliers even in small datasets~\citep{Cai2014robust}. 

\subsection{Our Contributions}
\label{sec:contribution}

Motivated by the considerations above, we consider a network model that accounts for the combined features. The clustering algorithm we consider is based on Semidefinite Programming (SDP) relaxation of the Modularity Maximization approach. We introduce a novel regularization term that penalizes outliers with unusual connection patterns beyond those implied by the inlier heterogeneity.

Two existing works serve as the foundation of our analysis: the Stochastic Block Model (SBM) with Outliers  in \cite{Cai2014robust} and the Degree-Corrected Stochastic Block Model (DCSBM) in \cite{CLX2017}.
As usual, the complexities arising by adding multiple realistic features largely surpass complexities of handling any of these features alone. In particular, our analysis needs to address the following challenges:

(1) When inliers are homogeneous, a node with unusual degree can be immediately recognized as an outlier. Therefore, for an outlier to hide, it must have a degree that is similar to inliers' degree. This  limits significantly the power of the adversary. In real networks, however, a node with unusual degree may not be an outlier -- it can well belong to one of the clusters, and have very high or low degree simply because this node is more popular/unpopular than other nodes in the same cluster. In this sense, we are facing the more challenging problem where the outliers have more freedom and do not need to restrict their degrees. The ability of the outliers to select across a broad range of degrees (especially when the inliers' degree distribution is heavy-tailed) makes it crucial for an algorithm to look into the detailed connectivity patterns of the nodes (who connects to whom) rather than just their degrees (how many they connect to), even more so than SBM.

(2) We use the primal-dual witness approach. However, in proving the necessary bounds for the recovered solution, we need to bound separately the contributions from nodes with different degrees. The heterogeneous nature of the nodes' connectivity complicates the analysis on the distribution of edges.   We need to obtain sufficiently tight individual bounds and ensure the correct dependence on the individual degrees such that the aggregate bound is sufficiently strong. In contrast, a worse case bound in terms of the maximum degree would be too loose. 
Moreover, in the degree-corrected set-up, the definiteness of the adjacency matrix becomes worse and the homogeneous penalization on diagonal terms is not enough to recover the true clusters. We instead introduce a term that  depends on the degree of each node, namely it takes the form $ \alpha \text{diag} \{d^*\} $, where $d^*$ is depends both on the degree vector $ d $ and a control on the expected number of edges to other inliers.

By addressing these points, we provide theoretical guarantees for the exact recovery of the inlier nodes with high probability. In particular, we impose no assumptions on the outlier nodes other than their cardinality. We provide an explicit and non-asymptotic condition for exact recovery. Namely, we request that the density gap (difference between the intra- and inter-cluster edge density) must be larger than an expression based on the natural problem parameters, such edge densities, amount of degree heterogeneity, sizes of the clusters, and number of inliers/outliers.
We also give explicit conditions on the tuning parameters of the algorithm.
Surprisingly, the condition for recovery does not contain an ``nm" term as in \citep{Cai2014robust} and instead contains two terms in ``n" and respectively "$\sqrt{n log n}$".

The applicability of our model is to networks of several thousands of nodes.
These are medium sized networks, which may arise in various applications and are subject  to the real-world features described above.
We provide numerical results based on synthetic data for a network of size in the range $n=400$ to $n= 1000$, divided into $r = 2$ clusters and in the presence of a varying number of outliers, $m \in [10, 30]$.
The degrees are following a heavy-tailed Pareto distribution, with varying shape parameters.
We compare the misclassification rate for our algorithm to state of the art algorithms, such as spectral clustering~\citep{ZLZ14}, SCORE~\citep{Jin2012} and Cai-Li~\citep{Cai2014robust}. Our results significantly improve the quality of recovery.
Notably, the performance of the algorithm is relatively unhindered even under very heterogeneous degree distributions and in the presence of a large number of outliers. In contrast, other algorithms have a sharp increase in the misclassification rate in such settings.

\subsection{Notation}
\label{sec:notation}

Matrices are denoted by bold capital letters, vectors by bold lower-case letters, and scalars by normal letter. The notation $\mtx{X} \succeq \mtx{0} $ means the matrix $\mtx{X}$ is positive semidefinite (psd). For two matrices $\mtx{X}$ and $\mtx{Y}$ of the same dimension, we denote their trace inner product by by $\langle \mtx{X}, \mtx{Y} \rangle := \text{Trace}(\mtx{X}^T \mtx{Y})$, and we write $\mtx{X} \le \mtx{Y}$ if $X_{ij} \le Y_{ij}$ for all $i$ and $j$. For an integer $k$, let $[k] := \{1,2,\ldots, k\}.$ We use $\mtx{I}$ to denote the identity matrix, $\mtx{J}$ the matrix with all entries equal to 1, and $\diag(\vct{u}) $ the diagonal matrix whose $i$-th diagonal entry is $u_i$. We use notations like $c, c_0, C$ etc.\ to denote numerical constants independent of the other model parameters (in particular, the number of nodes $n$). Finally, for two quantities $x \equiv x_n$  and $y \equiv y_n$ that may depend on $n$, we write $x \asymp y$ if they are of the same order, that is,  there exist numerical constant $c_1$ and $c_2$ such that $ c_1 y \le x \le c_2 y$.
 
\section{Problem Setup}
\label{sec:setup}

We consider a \emph{Degree-Corrected Stochastic Block Model with Outliers}, which is a generative model for a random graph with underlying clustering structures.

In particular, the model involves a graph $ \mathcal{G} = ( {V}, \mtx{A} ) $. Here $ {V} = [N] = [n+m] $ is a set of $ N:=n+m $ vertices, where $ n $ inliers are partitioned into $ r $ unknown clusters $ C_1^{\star}, C_2^{\star}, \cdots,  C_r^{\star} $, and the other $m$ nodes are outliers. The adjacency matrix $ \mtx{A} \in \{0,1\}^{(n+m)\times(n+m)} $,  where  $ {A}_{ij} = 1 $ if and only if nodes $ i $ and $ j $ are connected, are generated randomly as follows. Each pair of distinct inliers $ i \in C_a^* $ and $ j \in C_b^* $ are connected by an undirected edge with probability $ \theta_i \theta_j B_{ab} $, independently of all others. Here the vector $ \mtx{\theta} = (\v{\theta} )^{\top} \in \mathbb{R}_{+}^{n}$ is referred to as the degree heterogeneity parameters of the nodes. The symmetric matrix $ \mtx{B} = (B_{ab}) \in \mathbb{R}_{+}^{r\times r} $ is called the connectivity matrix of the clusters,  and specifies the likelihood of connectivity of the inliers. The connections of the $ m $ outliers among each other and to the inliers are arbitrary; they may depend on the underlying clusters and the realization of edges between the inliers, and may even be chosen adversarially. 

Note that the above model is a generalization of several well-known models. When there are no outliers ($m=0$) and  $\theta_i \equiv 1$ is uniform, the model reduces to the classical SBM~\citep{holland1983stochastic}. If $m=0$ but $\theta_i$ is allowed to vary across $i$, it becomes the  standard DCSBM~\citep{DHM2004,KN2011}. Finally, when $\theta_i\equiv 1$ but $m$ may be non-zero, it coincides with the setting considered in~\cite{Cai2014robust}, i.e., the SBM with outliers.

For future development, it is convenient to write the adjacency matrix $ A$ in a block form according to the clustering structure
\begin{equation} 
\label{eqn; semiraw}
\mtx{A}=\mtx{P}\begin{bmatrix}
\mtx{K} & \mtx{Z} \\
\mtx{Z^\transp} & \mtx{W}\\
\end{bmatrix}\mtx{P}^T
=\mtx{P}\begin{bmatrix}
\mtx{K_{11}} & \cdots & \mtx{K_{1r}} & \mtx{Z_1}\\
\vdots  &  \ddots  & \vdots\ & \vdots\\
\mtx{K_{1r}^\transp}   & \cdots & \mtx{K_{rr}}  & \mtx{Z_r}  \\
\mtx{Z_1^\transp} & \cdots  & \mtx{Z_r^\transp}\ & \mtx{W}\\
\end{bmatrix}\mtx{P}^T,
\end{equation}
where the block matrices above have the following interpretations:
\begin{itemize}
\item $ \mtx{W} \in \{0,1\}^{m\times m} $ is a symmetric 0-1 matrix representing the connection within the outliers. Under our model, $\mtx{W}$ is arbitrary.
\item  $ \mtx{Z} \in \{0,1\}^{n\times m} $ is a 0-1 matrix representing the connection between inliers and outliers; in particular, $\mtx{Z_a}$ is the adjacency matrix between the outliers and the $a$-th inlier cluster. Under our model, $\mtx{Z}$ is arbitrary.

\item  $ \mtx{K} \in \{0,1\}^{n\times n} $ is a symmetric 0-1  matrix representing the connection between inliers. In particular, $\mtx{K_{ab}}$ is the adjacency matrix between the $a$-th and $b$-th clusters. Under our model, each entry of $\mtx{K_{ab}}$ is equal to 1 with probability $ \theta_i \theta_j B_{ab} $, independently of all others. 
\item $ \mtx{P} \in \{0,1\}^{N\times N}$ is an unknown  permutation matrix, in which there is a single 1 in each row and column while all other entries are 0. Under this permutation, the nodes are ordered according to the underlying structure of clusters and outliers.
\end{itemize}
For each $a\in[r]$, we denote the size of the $a$-th cluster  by  be $ l_a := |C^*_a| $. Note that $ n = \sum_{a=1}^{r} l_a$. Let $ l_{\min} = \min_{1\le a \le r}l_a $  be the minimum size of the clusters. We also introduce the vector of node degrees $ \mtx{d} = ( d_1, \cdots, d_{n+m} )^\top $, where  $ d_i := \sum\limits_{j=1}^{n+m} A_{ij} $.

For each candidate partition of the $ (n+m) $ inliers into several clusters, we may associate it with a \emph{partition matrix} $ \mtx{X} \in \{0, 1\}^{(n+m)\times(n+m)}$, such that $ X_{ij} = 1$ if and only if nodes $ i $ and $ j $ are assigned to the same cluster, with the convention that $ X_{ii} = 1 $. Ideally, we would like to find a partition matrix of the form
\begin{align}
\label{eq:good_partition}
\mtx{X} =
\mtx{P}^T
\begin{bmatrix}
\mtx{J}_{l_1} & \mtx{0} & \cdots & \mtx{0} & * \\
\mtx{0} & \mtx{J}_{l_2} & \cdots &\mtx{0} & * \\
\vdots  &  \vdots  & \ddots\ & \vdots &  \vdots \\
\mtx{0}   & \mtx{0} & \cdots & \mtx{J}_{l_r}  & *  \\
* & * & \cdots  & *  & *\\
\end{bmatrix}
\mtx{P}^T,
\end{align}
where $\mtx{J}_l$ denotes the $l$-by-$l$ all-one matrix, and $*$ denotes arbitrary entries. In other words, we want to correctly recover the cluster structure within the inliers, where cluster assignment of the outliers may be arbitrary.

Given a single realization of the resulting
random graph $ \mathcal{G} = ( V, \mtx{A} ) $, our goal is to recover the true inlier clusters $ \left\{C_a^*\right\}_{a=1}^{r} $, that is, to recover a partition matrix in the form of \eqref{eq:good_partition}.

\section{Algorithm: A Convex Relaxation Approach}
\label{sec:algorithm}

In this section, we provide the motivation and description of our algorithm, which can handle both degree heterogeneity and outliers.

We begin by recalling that a classical approach to clustering nodes in a network is \emph{modularity maximization}~\citep{Newman2006modularity}, which involves solving the optimization problem
\begin{equation}
\label{eq:mm}
\begin{aligned}
&\min_{\mtx{X} \in \mathbb{R}^{N\times N}} && \langle \mtx{X}, \;  \lambda \mtx{d}\mtx{d}^\transp - \mtx{A}\rangle
\\
&\text{subject to} && \mtx{X} \text{ is a partition matrix.}
\end{aligned}
\end{equation}
The negative of the objective of the above optimization problem is called \emph{modularity}, which is a measure of the quality of the candidate clustering $\mtx{X}$. This quality measure is derived and studied in depth in the work of~\cite{Newman2006modularity}, which shows that maximizing the modularity provides a natural and robust framework for finding a good clustering of the nodes. 

In general, the modularity optimization problem~\eqref{eq:mm} is intractable due to the need of searching over partition matrices, a non-convex and combinatorial constraint.
Replacing this constraint with a convex constraint, \cite{CLX2017} propose the following convex, SDP relaxation of modularity maximization:
\begin{equation}
\label{eq:cvxraw}
\begin{aligned}
&\min_{\mtx{X}\in \mathbb{R}^{N\times N}} && \langle \mtx{X}, \;  \lambda \mtx{d}\mtx{d}^\transp - \mtx{A}\rangle
\\
&\text{subject to} && \mtx{X}\succeq \mtx{0},
\\
&~              && \mtx{0} \leq \mtx{X} \leq \mtx{J},
\end{aligned}
\end{equation}
where we recall that $ \mtx{J} $ is the $ n \times n $ all-one matrix. They provide recovery guarantees for the above convex relaxation under DCSBM \emph{without} outliers.
On the other hand, to handle outliers in the classical SBM setting, \cite{Cai2014robust} propose a convex relaxation formulation that penalizes the diagonal entries of $ X $:
\begin{equation}
\label{eq:cvx}
\begin{aligned}
&\min_{\mtx{X}\in \mathbb{R}^{N\times N}} && \langle \mtx{X}, \; \alpha \mtx{I} + \lambda \mtx{J} - \mtx{A}\rangle
\\
&\text{subject to} && \mtx{X}\succeq \mtx{0},
\\
&~              && \mtx{0} \leq \mtx{X} \leq \mtx{J}.
\end{aligned}
\end{equation}
This formulation, however, is unable to handle DCSBM as it treats all nodes equally without considering the variation in their degrees.

\textbf{Our Algorithm:} Building on the formulations~\eqref{eq:cvxraw} and~\eqref{eq:cvx}, we propose a convex relaxation formulation that accounts for both degree heterogeneity and outliers. Given that  outliers can have any degree, we need to incorporate a larger penalization term on diagonal entries of $ X $. In particular, we penalize a potential outlier whose degree exhibits unusual behavior beyond the normal variation implied the DCSBM.

To be more specific, our algorithm depends on several quantities of the model. For each true cluster $ C_a^* $, we define the aggregate degree heterogeneity parameter as
\[
G_a := \sum_{i\in C_a^*} \theta_i.
\]
Consequently, the expected number of edges from each node $i$  to other inliers is equal to $ \theta_i H_a$, where
\[H_a := \sum_{b=1}^{r} G_b B_{ab}.
\]
Our diagonal penalization term is based on the quantity $ d_i^* := \max\left\{ d_i, H^+\right\} $, where $ H^+ := \max_{1\le a\le r} H_a $. 
With the above notation, we consider the following convex relaxation formulation
\begin{equation}
\label{eq:cvxnew}
\begin{aligned}
&\min_{\mtx{X}} && \langle \mtx{X}, \; \alpha \cdot \diag(\mtx{d^{*}}) + \lambda \mtx{d}\mtx{d^\transp} - \mtx{A}\rangle
\\
&\text{subject to} && \mtx{X}\succeq \mtx{0},
\\
&~              && \mtx{0} \leq \mtx{X} \leq \mtx{J}.
\end{aligned}
\end{equation}
One can see that our formulation is a convex relaxation of modularity maximization with an additional node-dependent regularization term on the diagonal entries of $\mtx{X}$. In particular, we penalize each node $i$ differently with the weight $d_i^*$, which is an upper bound of the node's degree $ d_i $ that also captures the positive deviation from the expected connections to other inliers. 
The tuning parameter $ \alpha $ controls the strength of this diagonal penalization, and should be chosen to be sufficiently large.
In our theoretical results in the next section, we provide guidance on how to choose $ \alpha $; in particular, we need $ \alpha \ge c_1 \frac{m}{H^-} $, where $ H^- := \min_{1\le a\le r} H_a $ and $c_1$ is a numerical constant.


\section{Theoretical Guarantees}

In this section, we provide theoretical guarantees on the performance of the convex optimization approach~(\ref{eq:cvxnew}) under the setting of DCSBM with Outliers described in Section~\ref{sec:setup}. Before stating our main theorem, we introduce several quantities of interest, and record some useful relationships between them. 

\subsection{Additional Notations and Preliminary Facts}

We first provide a summary of the notations used in the sequel. Without loss of generality, assume that the first $n$ nodes, $\{1,2,\ldots, n\}$, are inliers.
\begin{itemize}
	\item $ p^+ := \max_{1\le a\le r} B_{aa},$ and $ ~p^- := \min_{1\le a\le r} B_{aa} $
	\item $ q^+ := \max_{1\le a < b\le r} B_{ab},$ and $ ~q^- := \min_{1\le a<b\le r} B_{ab} $
	\item $ \theta_{\min} := \min_{1\le i \le n} \theta_i $.
	\item $ G_a := \sum\limits_{i\in C_a^*} \theta_i$, $ H_a := \sum\limits_{1\le b \le r} B_{ab}G_b $,  $ H^+ := \max\limits_{1\le a \le r} H_a$, and $ H^- := \min\limits_{1\le a \le r} H_a$, as defined previously.
	\item $ \bar{\theta} := \sum\limits_{i=1}^{n} \theta_i / n$, $ \bar{\theta}_a := \frac{G_a}{l_a} $, $ \bar{\theta}_{\min} = \min\limits_{1\le a \le r} \bar{\theta}_a $, and $G_{\min} = \min_{1\le a\le r} G_a .$
	\item $ f_i := \theta_i H_a $, which is the expected degree of $ i $-th vertex with inliers.
	\item $ \tilde{d}_a := \sum\limits_{i \in C_a^*} d_i $, which is the sum of the degrees of nodes in cluster $ C_a^* $.
	\item  For a matrix $\mtx{T}$ and each pair $ 1 \le a, b \le r$, we use $ \mtx{M}_{(a, b)} \in \mathbb{R}^{l_a \times l_b} $ to denote the submatrix of $ \mtx{M} $ with entries indexed by $ C_a^* \times C_b^*$.
\end{itemize}
By definition, it is clear that 
$$ G_a \ge l_{\min} \bar{\theta}_{\min} 
\qquad \text{and} \qquad 
n \bar{\theta} q^- \le H_a \le n \bar{\theta} p^+ $$
for all $ 1 \le a \le r.$
Note that the expected degrees of inliers is determined by the quantities $\theta_i $ and $ H_a $; in fact, we have $ \mathbb{E} \sum_{1\le j \le n} A_{ij} = f_i := \theta_i H_a $.


\subsection{Guarantee for Perfect Clustering}
We are now ready to state the main result of the paper. Recall that our goal is to find a partition matrix of the form~\eqref{eq:good_partition} given the adjacency matrix $\mtx{A}$, that is, to recover the cluster structure of the inliers from a single realization of a graph generated from DCSBM with outliers. The theorem below, proved in Section~\ref{sec:proof_thm:exact}, provides sufficient conditions for when our convex relaxation approach~\eqref{eq:cvxnew} achieves this goal.

\begin{theorem}
	\label{thm:exact}
	Assume that $ p^+ \asymp  p^- \asymp  q^+ \asymp q^- $  and $ \bar{\theta}  \asymp \bar{\theta}_{\min} $. 
	Suppose that $ q^- \ge \frac{m}{l_{\min}} $ and
	\begin{equation} \label{eqn: delta}
	\delta \ge c_0 \left\{ \sqrt{\frac{p^+ \log n}{\theta_{\min} G_{\min}}} + \frac{\alpha n \bar{\theta} p^+}{G_{\min}} + \frac{\theta_{\max} \sqrt{p^+ n \log n}}{G_{\min}\theta_{\min}} + \frac{\log n}{G_{\min}\theta_{\min}} + \frac{m\sqrt{r}}{\theta_{\min} G_{\min}} + \frac{m}{\alpha \theta_{\min} G_{\min}} \right\}
	\end{equation}
	for some $\delta>0$, and that the tuning parameters in~\eqref{eq:cvxnew} satisfy 
	\begin{equation} \label{eqn:lambda}
	\max\limits_{1\le a < b \le r} \frac{B_{ab}+\delta}{H_a H_b} < \lambda < \min\limits_{1 \le a \le r} \frac{B_{aa}-\delta}{H_a^2}
	\end{equation} 
	and 
	\begin{equation} \label{eqn:alpha}
	\alpha \ge c_1 \frac{m}{H^-},
	\end{equation}
	where $c_0, c_1>0$ are sufficiently large numerical constants.
	Then with probability at least $ 1 - \frac{1}{n} - \frac{2r}{n^2} -\frac{cr}{l^4_{\min}} $ for some constant $ c $, any solution $ \mtx{\hat{X}} $ to the semidefinite program  (\ref{eq:cvxnew}) must be of the form 
	\begin{equation}
	\mtx{\hat{X}} = \mtx{P}	\begin{bmatrix}
	\mtx{J_{l_1}} &&&\mtx{\hat{Z_1}}\\
	&\ddots&&\vdots\\
	&&\mtx{J_{l_r}}&\mtx{\hat{Z_r}}\\
	\mtx{\hat{Z_1}}^T&\cdots&\mtx{\hat{Z_r}}^T&\mtx{\hat{W}}
	\end{bmatrix}\mtx{P}^T,
	\end{equation}
	where $ \mtx{P} $ is a permutation matrix.
\end{theorem}

Theorem \ref{thm:exact} guarantees that any optimal solution $ \mtx{\hat{X}} $ satisfies the property that for any inliers $ i $ and $ j $, $ \hat{X}_{ij} = 1 $ if nodes $ i $ and $ j $ are in the same true cluster and $ \hat{X}_{ij} = 0 $ otherwise. In other words, $\mtx{\hat{X}} $  correctly recovers the true cluster structure of the inliers.
Since we impose no assumption on the outliers, there is in general no hope of determining how outliers would be clustered. Consequently, the theorem does not provide guarantees on the values of the elements on the last $ m $ rows and $ m $ columns of $ \mtx{\hat{X}} $. Nevertheless, the theorem ensures that the presence of the outliers does no hinder the clustering of the inliers.  

Once we obtain the solution $ \mtx{\hat{X}} $ as above, we can extract from it an explicit clustering of the inliers by treating each row of $\mtx{\hat{X}} $ as a point in $\mathbb{R}^{n+m}$ and running the $k$-means algorithm; see~\cite{Cai2014robust,CLX2017} for the details. 

The results in Theorem~\ref{thm:exact} are non-asymptotic and valid for finite $n$; in particular, the probability for recovery has the form $ 1 - O(\frac{1}{n}) $, which is the same as in \citet[Theorem 3.1]{Cai2014robust} and \citet[Theorem 3.3]{CLX2017}. Let us parse the recovery condition in Theorem~\ref{thm:exact} under the simplified setting with $p^+=p^-=p$, $q^+ = q^- = q$, and $l_a = l_{\min}, \forall 1\le a \le r$; that is, the connectivity matrix $\mtx{B}$ has diagonal entries all equal to $p$ and off-diagonal entries all equal to $q$, and all clusters have the same size $l_{\min}$
\begin{itemize}
    \item First consider the special case where the node degrees are uniform (no degree heterogeneity); that is, $\theta_i = 1,\forall 1\le i \le n.$ In this case, noting that $p\asymp q$ by assumption and performing some algebra, we find that the conditions~\eqref{eqn: delta}--\eqref{eqn:alpha} simplify to
    \begin{gather*}
        \delta \ge c_0 \left\{ \sqrt{\frac{p \log n}{ l_{\min}}} + \frac{\alpha n p}{l_{\min}} + \frac{\sqrt{n  q \log n}}{l_{\min}}  + \frac{m\sqrt{r}}{ l_{\min}} + \frac{m}{\alpha  l_{\min}} \right\}, \\
         \frac{q+\delta}{f^2} < \lambda < \frac{p-\delta}{f^2}, \\
        \alpha \ge c_1 \frac{m}{f},
    \end{gather*}
    where $f := qn + (p-q) l_{\min} $ is the expected inlier degree. Up to a rescaling by $f$, these conditions match those in \citet[Theorem 3.1]{Cai2014robust} under the same setting.
    \item Next consider the special case where there is no outliers; that is, $m = 0$. In this case, we may take $\alpha = 0$; moreover, by again noting that $p\asymp q$ and performing some algebra, we find that the conditions~\eqref{eqn: delta}--\eqref{eqn:lambda} become
    \begin{gather*} 
    	\delta \ge c_0 \left\{ \sqrt{\frac{p \log n}{\theta_{\min} G_{\min}}} + \frac{\sqrt{qn \log n}}{G_{\min}} \cdot \frac{\theta_{\max} }{\theta_{\min} } \right\} , \\
    	\max\limits_{1\le a < b \le r} \frac{q+\delta}{H_a H_b} < \lambda < \min\limits_{1 \le a \le r} \frac{p-\delta}{H_a^2}.
	\end{gather*} 
    These conditions match those in~\citet[Theorem 3.3]{CLX2017} except for an addition term $\frac{\theta_{\max} }{\theta_{\min} }$ in the gap condition for $\delta$.
\end{itemize}
Therefore, in the special cases of SBM with oultiers and DCSBM, we see that Theorem~\ref{thm:exact} is strong enough to essentially recover the results in \cite{Cai2014robust,CLX2017} as corollaries. Moreover, Theorem~\ref{thm:exact} strictly generalizes their results as it is applicable in the setting with both outliers and degree heterogeneity.

\section{Experiments}

In this section, we provide numerical experiment results demonstrate the performance of our algorithm for clustering heterogeneous networks with outliers. We also compare our algorithm with several state-of-the-art algorithms.

Recall the structure of the adjacency matrix $ \mtx{A} $ as given in equation~\eqref{eqn; semiraw}, which we reproduce below
\begin{align}
\mtx{A}=\mtx{P}\begin{bmatrix}
	\mtx{K} & \mtx{Z} \\
	\mtx{Z^\transp} & \mtx{W}\\
\end{bmatrix}\mtx{P}^T.
\end{align}
With this in mind, we now describe how we generate the inlier part $ \mtx
K $ and the outlier part $ (\mtx{Z} , \mtx{W} )$ of the adjacency matrix.

\paragraph{Inliers:}

For  each inlier node $ i\in[n] $, the degree heterogeneity parameter $ \theta_i $ is sampled independently from a Pareto($ \alpha $, $ \beta $) distribution with the density function $ f(x|\alpha, \beta) = \frac{\alpha \beta^\alpha}{x^{\alpha+1}} \mathbf{1}_{\{x \ge \beta\}} $, where $ \alpha $ and $ \beta $ are called the \emph{shape} and \emph{scale} parameters, respectively.  We consider different values of the shape
parameter, and choose the scale parameter accordingly so that the expectation of each $ \theta_i $ is fixed at $ 1 $. Note that the heterogeneity of the degree $ \theta_i $’s decreases as the shape parameter $ \alpha $ increases. Given the above $ \vct{\theta} $ and two given  inter and intra-cluster density parameters $ 0<q<p<1 $, we then generate $ \mtx{K} $ according to DCSBM with parameters $ p $, $ q $ and the $ \vct{\theta} $. 

\paragraph{Outliers:}

For generating the outliers we follow~\cite[pp.\ 7]{Cai2014robust}. Let $ \tau \in [0,1]$ be a fixed number. We assume that for each $ i\in [n] $ and $ j\in [m] $, $ Z_{ij} \sim \text{Bernoulli}(\rho_i \tau) $ and $ \sqrt{\rho_i} \sim \text{Uniform}(0,1) $. We also assume that for each $ 1 \le i < j \le m $, $ W_{ij} \sim \text{Bernoulli}(0.7\tau) $. Here $ \tau $ controls the degrees of the outliers.\\

In the following experiments, we choose the parameter $\tau$ such that the outliers' expected degree is moderately above the average of the  inliers' degrees. Given that the inliers' degrees are heavy-tailed, this means that the outlier's degrees are not distinguishable from inliers with a larger degree.
The larger the $\tau$ parameter, the harder is the recovery problem.

In Figure~\ref{fig:400_dense_outliers} we show the performance of our algorithm in terms of the misclassification rate. 
Here we consider varying values for the shape parameter, the number of outliers and intra-cluster density parameter $ p $.
The inter-cluster density parameter is $q = p / 3$. 
As can be seen from the figure, as the problem gets harder in terms of more heterogeneity, more outliers or more sparsity, the performance of our algorithm degrades gracefully.
For $p$ as low as $20\%$ we note that the performance suffers only very little as the degree distribution gets significantly heavier (as captured by a shape parameter $\alpha$) and as we increase the number of outliers.
Very sparse graphs (with intra-cluster connectivity $p = 8\%$ are naturally more sensitive.

\begin{figure}
	\centering
	\begin{tabular}{cc}
		\includegraphics[width=0.45\columnwidth, clip, trim=0 0 0 0]{\figpath 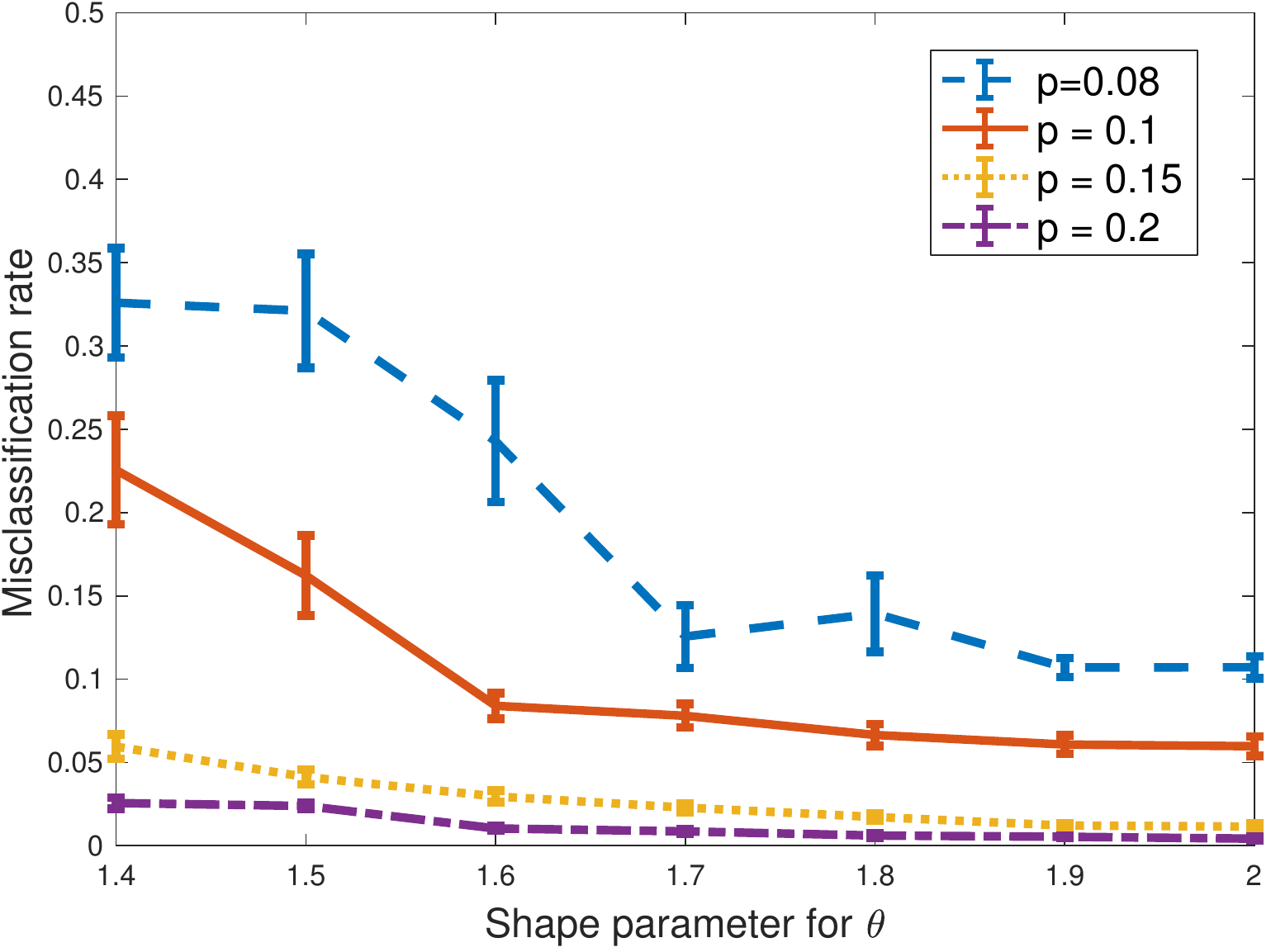}&  
		\includegraphics[width=0.45\columnwidth, clip, trim=0 0 0 0 ]{\figpath 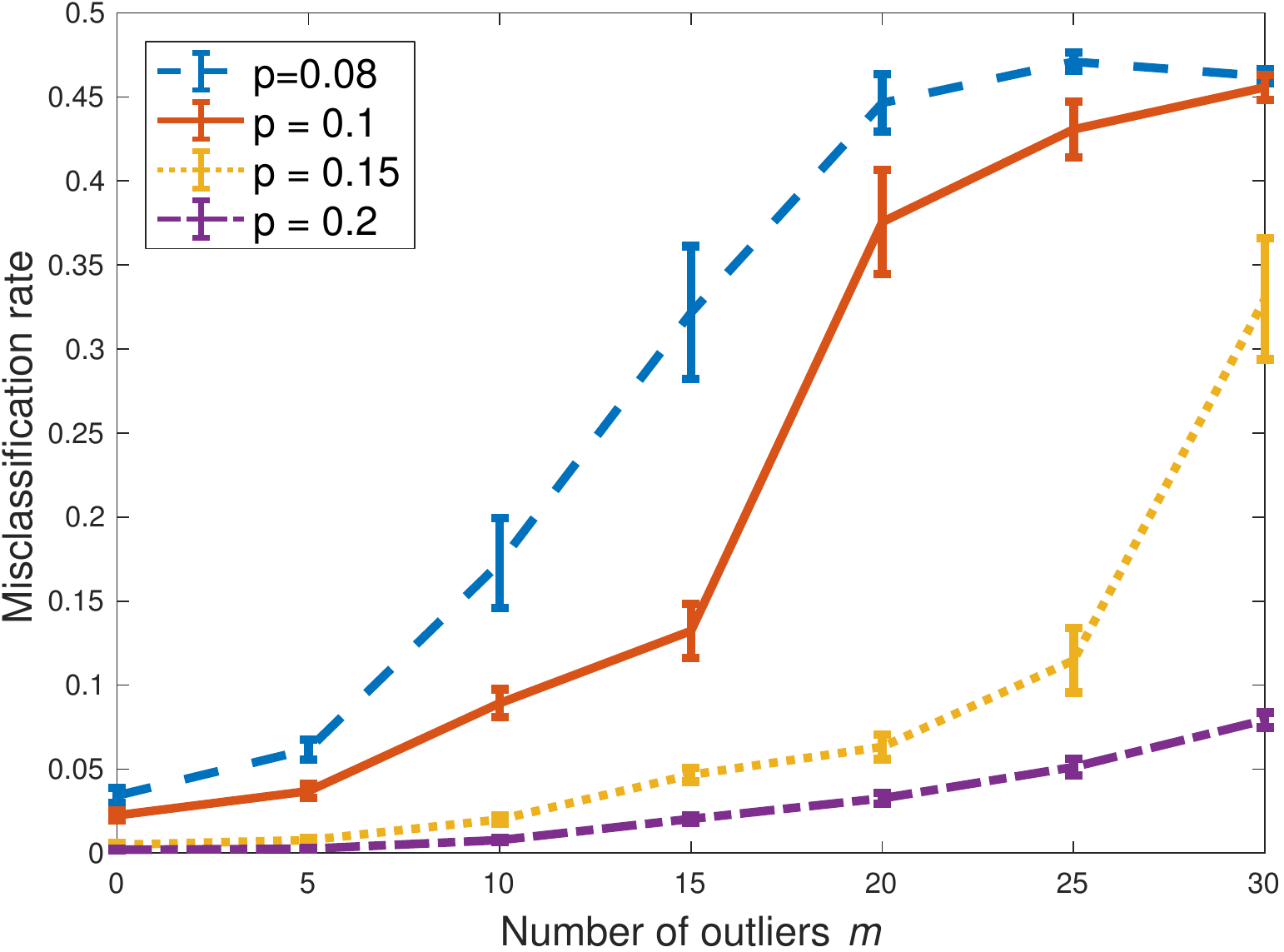}
		\\
	\end{tabular}
	\caption{Misclassification rate versus $ p $, variability of $ \vct{\theta} $ and number of outliers. Here $ n=400$, 2 equal-sized clusters, $ q=p/3 $, and $ \tau = 1 $.  Each point is the average of $ 20 $ trials. Left: $ m=10 $. Right: shape parameter = $ 1.7 $.
	}
	\label{fig:400_dense_outliers}
\end{figure}

In Figure~\ref{fig:1000_dense_outliers}, we consider a setting similar to Figure~\ref{fig:400_dense_outliers}, but with larger graphs $ n=1000 $. The results demonstrate the same relatively unhindered performance under increased heterogeneity and number of outliers, when the graph is not too sparse. 

\begin{figure}
	\centering
	\begin{tabular}{cc}
		\includegraphics[width=0.45\columnwidth, clip, trim=0 0 0 0]{\figpath 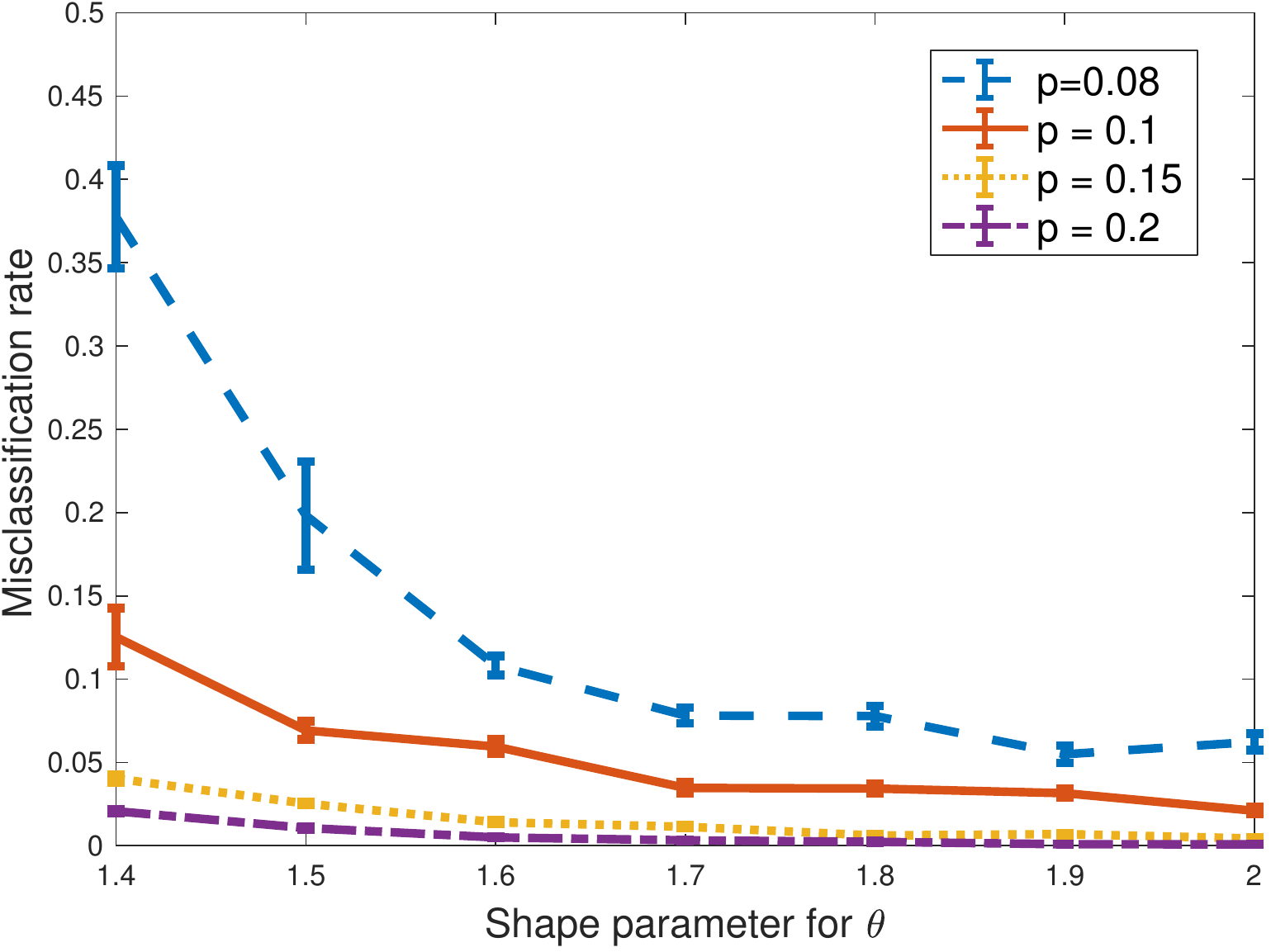}&  
		\includegraphics[width=0.45\columnwidth, clip, trim=0 0 0 0 ]{\figpath 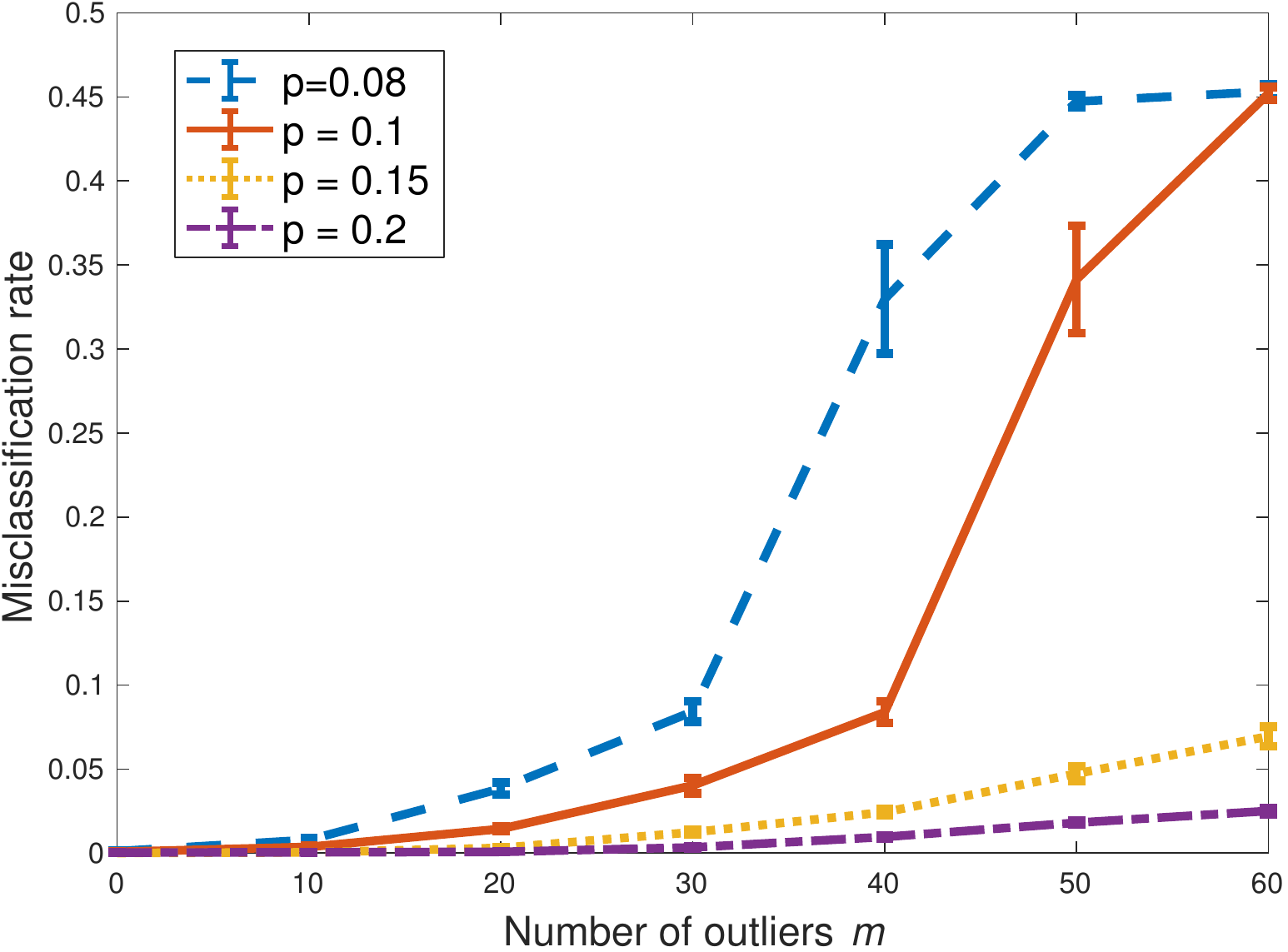}
		\\
	\end{tabular}
	\caption{(Larger graphs.) Misclassification rate versus $ p $, variability of $ \vct{\theta} $ and number of outliers. Here $ n=1000$, 2 equal-sized clusters, $ q=p/3 $, and $ \tau = 1 $.  Each point is the average of $ 20 $ trials. Left: $ m=30 $. Right: shape parameter = $ 1.7 $.}
	\label{fig:1000_dense_outliers}
\end{figure}

We next decrease the connectivity of the outliers, as we set $ \tau =0.5 $.
In this case, the problem becomes easier, as outliers are more restricted.
As shown in Figure~\ref{fig:400_sparse_outliers}, the misclassification rates decrease and remain small even as we increase the number of outliers and the heterogeneity of the inliers.

\begin{figure}
	\centering
	\begin{tabular}{cc}
		\includegraphics[width=0.45\columnwidth, clip, trim = 0 0 0 0]{\figpath 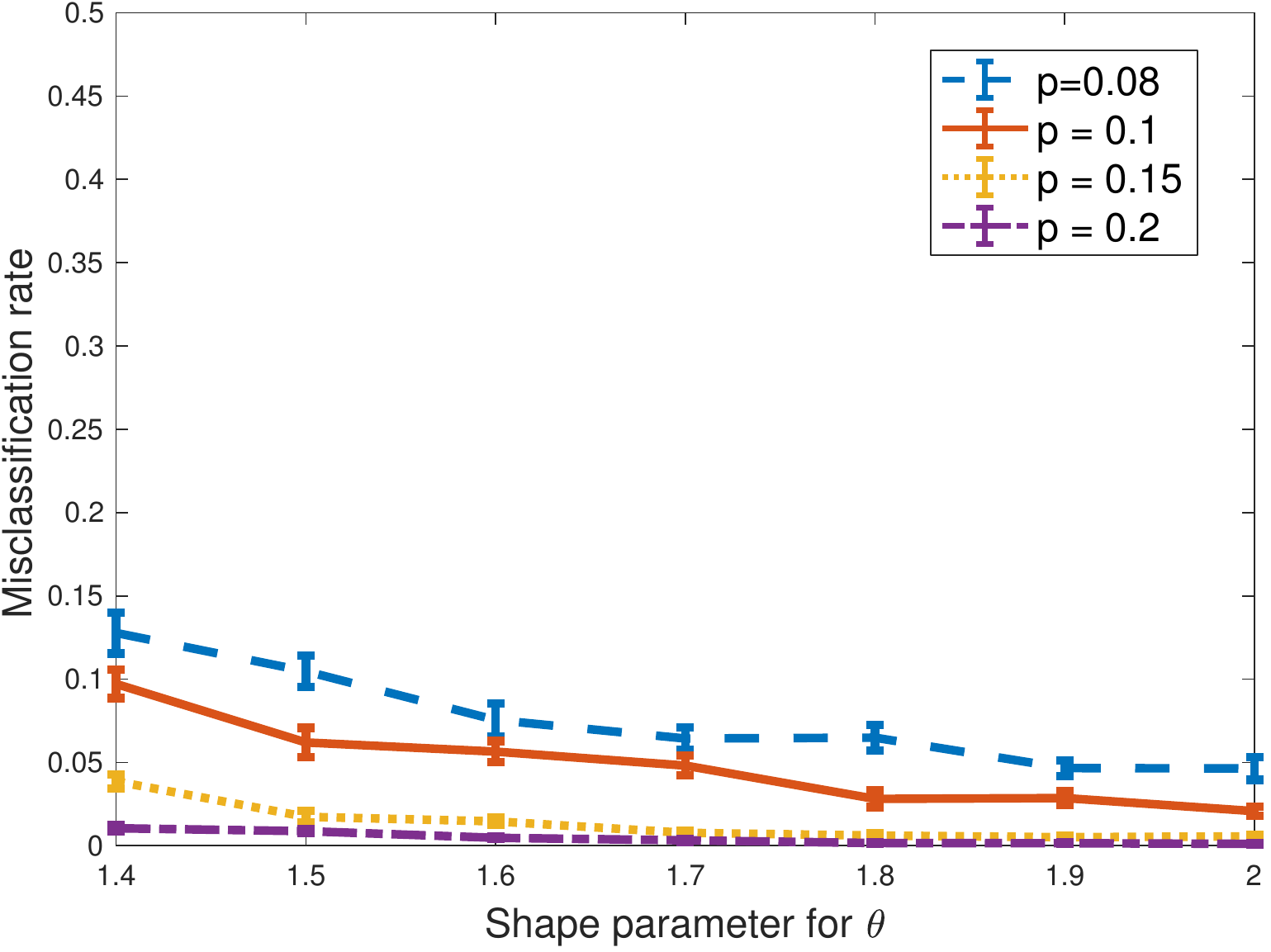} &
		\includegraphics[width=0.45\columnwidth, clip, trim = 0 0 0 0]{\figpath 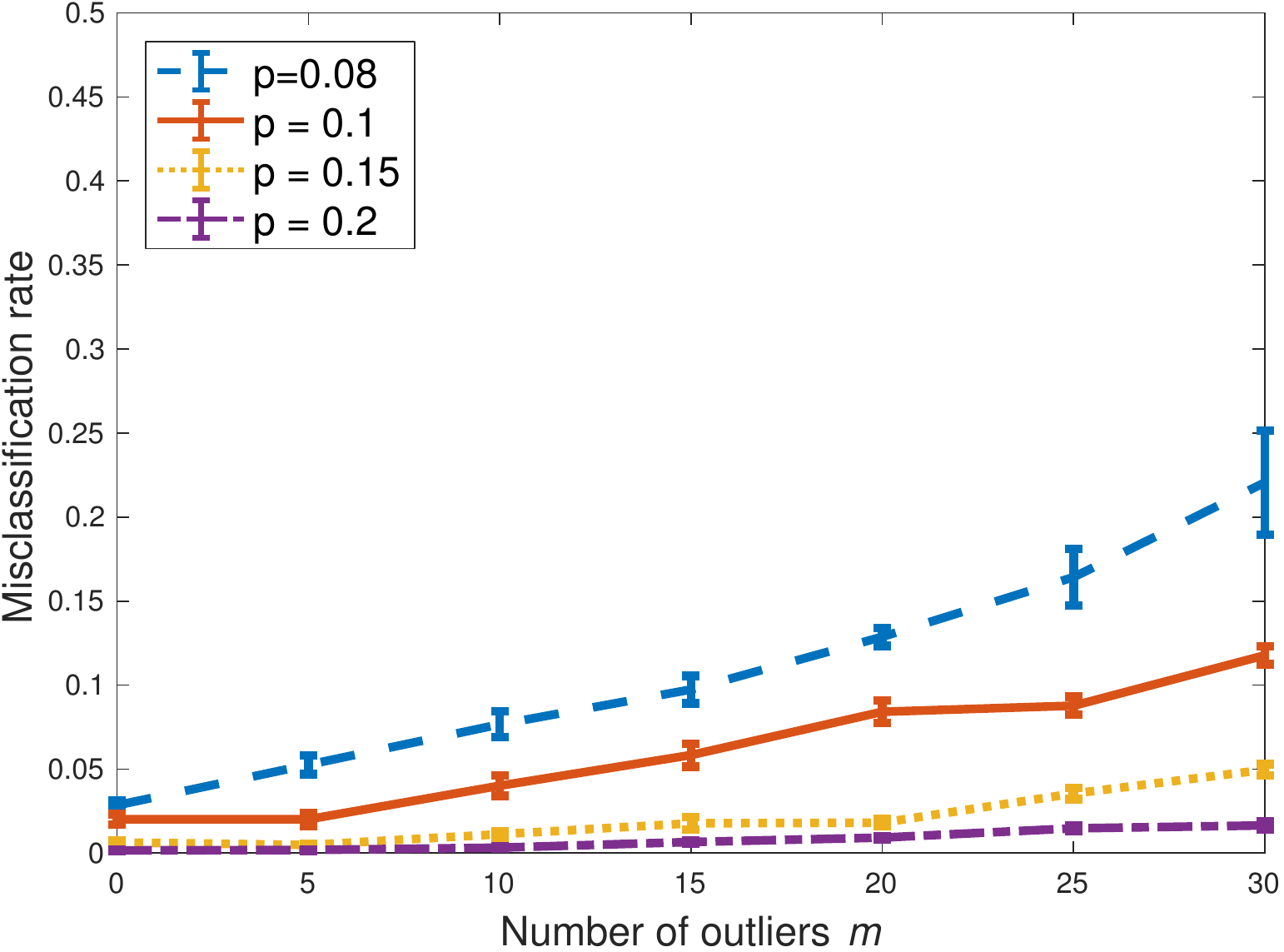}
	\end{tabular}
	\caption{(Sparser outliers.) Misclassification rate versus $ p $, variability of $ \vct{\theta} $ and number of outliers. Here $ n=400$, 2 equal-sized clusters, $ q=p/3 $, and $ \tau = 0.5 $.  Each point is the average of $ 20 $ trials. Left: $ m=10 $. Right: shape parameter = $ 1.7 $. }
	\label{fig:400_sparse_outliers}
\end{figure}

Finally, in Figure~\ref{fig:compare},  we compare our algorithm with three state-of-the-art algorithms:  spectral clustering~\citep{ZLZ14}, SCORE~\citep{Jin2012} and Cai-Li~\citep{Cai2014robust}.
The gain in performance is significant, and in particular for the more adversarial settings with high degree.

\begin{figure}
	\centering
	\begin{tabular}{cc}
	\includegraphics[width=0.45\columnwidth, clip, trim = 0 0 0 0]{\figpath 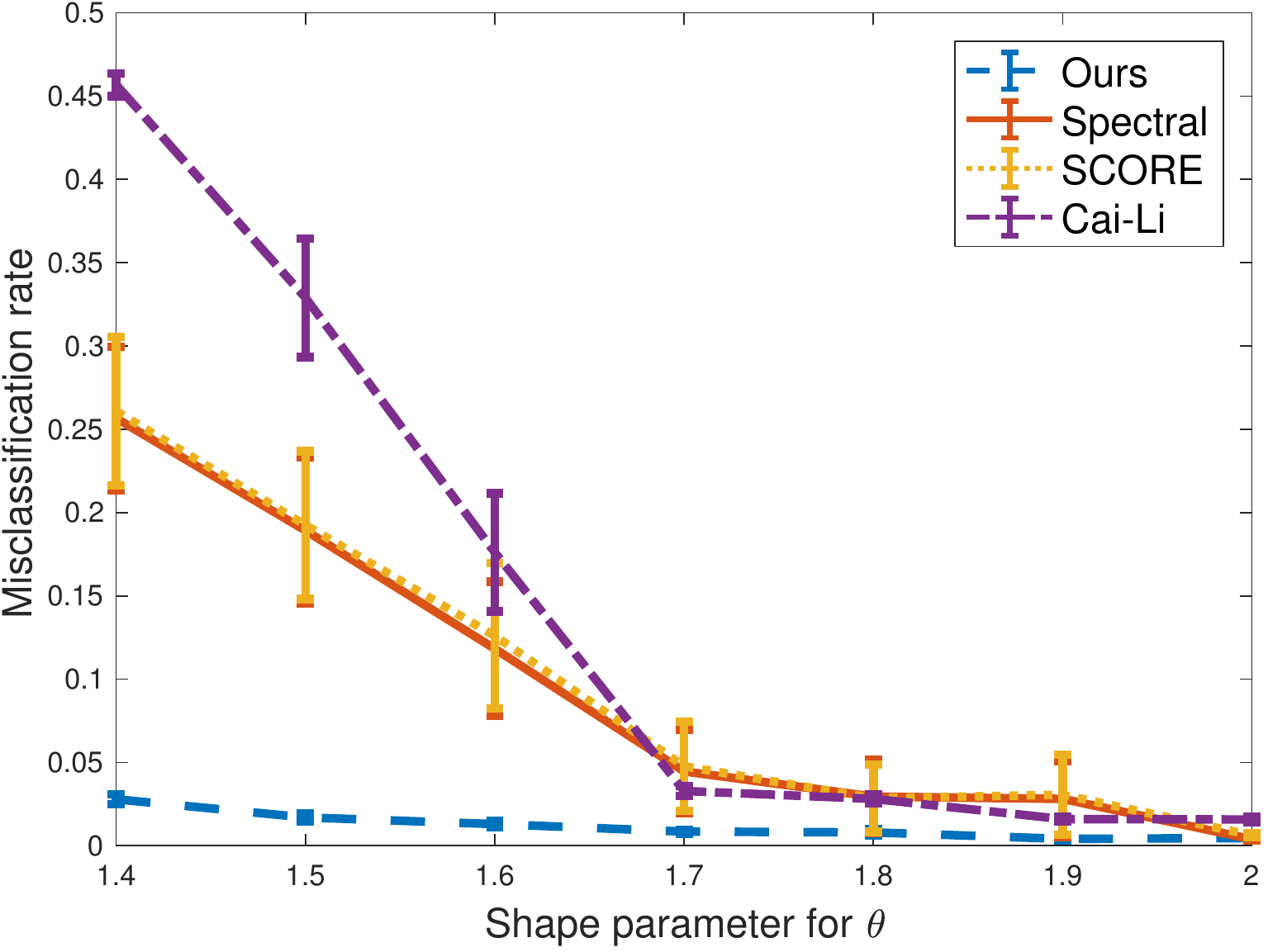} &
	\includegraphics[width=0.45\columnwidth, clip, trim=0 0 0 0]{\figpath 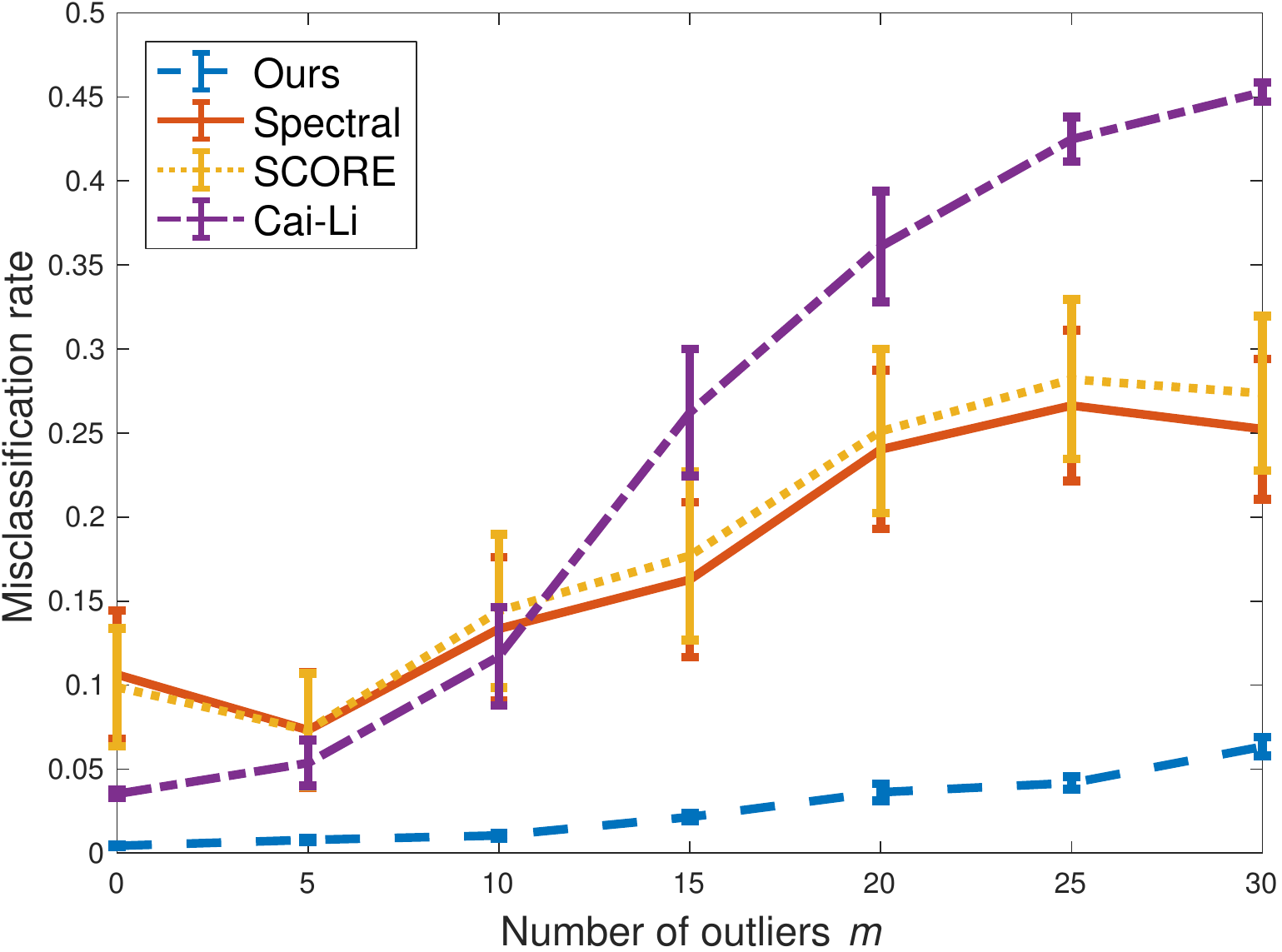}
	\end{tabular}
	\caption{Comparison with spectral clustering, SCORE and Cai-Li. Here $ n=400 $, 2 equal-sized clusters, $ p=0.15 $,  $ q=p/3 $, $ \tau = 0.5 $, and each point is the average of $ 20 $ trials. Left: Misclassification rate versus variability of $ \vct{\theta} $, with $m=10 $ outliers. Right:  Misclassification rate  versus number of outliers $ m $, where the shape parameter for $ \vct{\theta} $ is $ 1.6 $. The standard error bars are shown.
	}
	\label{fig:compare}
\end{figure}



\section{Proof of Theorem~\ref{thm:exact}}\label{sec:proof_thm:exact}

In this section, we prove our main result in Theorem~\ref{thm:exact}.

\subsection{Roadmap of the Proof}

The high level strategy of the proof involves using a \emph{primal-dual witness} approach, which  consists of two steps:
\begin{enumerate}
	\item We first construct a candidate optimal \emph{primal} solution to the convex program~\eqref{eq:cvxnew}. This is done by solving an auxiliary optimization problem; see Lemma~\ref{lemma:1}.
	\item We then certify that this candidate solution is indeed optimal by showing that it satisfies a form of the first-order optimality (KKT) condition, which involves the existence of a corresponding \emph{dual} solution/certificate. This is done by explicitly constructing the dual certificate and proving that it has the desired properties with high probability. A crucial step in the analysis is to decompose the penalized connecting matrix $ \alpha \diag\left(\mtx{d^{*}}\right) + \lambda \mtx{d d^T} - \mtx{A} $ into four terms and establish high-probability bounds for each of  them.
\end{enumerate}

The reason for using the above strategy is as follow: Our goal is to recover the true inlier clusters, so the ``inlier part'' of the desired solution should have a block-diagonal form that corresponds to ground truth clusters, as in equation~\eqref{eq:good_partition}. However, a priori we do not know what the ``outlier part'' of the solution will look like --- it depends on the edge connection of the outliers, and in general will not be exactly zero. Therefore, we need to first ``pin down'' the outlier part of the solution, which is precisely the Step 1 above. To show this solution is indeed optimal, we prove that there exists a corresponding dual solution that ``certifies'' its optimality, which is the goal of the Step 2 above. Below we elaborate on the main technical challenges and novelty in these two steps.

In Step 1, we construct a candidate solution $ \mtx{X^{\star}} = \mtx{V^{*} V^{*^{T}}} $ that is feasible to the primal problem. A major difficulty of proving the optimality of $\mtx{X^{\star}}$ is in that a priori we do not know the exact value the matrix $ \mtx{X^{\star}} $. To overcome this difficulty, we note that the candidate solution $\mtx{X^{\star}}$ is constructed from the optimal solution of the auxiliary optimization problem. The KKT condition of the auxiliary optimization problem gives several desirable constraints for the outlier parts of its primal and dual solutions (i.e., the constraints on $ \beta $ and $ \mtx{x_a} $ in Lemma~\ref{lemma:1}); in particular, the solution $\sum_{a=1}^r \mtx{x_a} \mtx{x_a}^T$  must be perpendicular to the normal vector of the semidefinite cone constraint. We show that this property is equivalent to $ \mtx{\Lambda V^{*}} = \mtx{0} $, where  $\mtx{\Lambda}$ is the outlier part of the matrix $ \mtx{E}=\alpha \diag\left\{\mtx{d^{*}}\right\} + \lambda \mtx{d}\mtx{d^\transp} - \mtx{A} $ that appears in the objective of our convex relaxation approach~\eqref{eq:cvxnew}; cf.~\eqref{def:sterms}. This property allows us to understand the effect of outlier part of the solution $\mtx{X^{\star}}$ and subsequently find the closed form of other parts. 

In Step 2,  to establish the optimality of $\mtx{X^{\star}}$, we need to show that it has an objective value no larger than that of any other feasible solution $ \mtx{X} $. In other words, we need to show that $ \Delta \left(\mtx{X}\right) \triangleq \langle \mtx{X}^*-\mtx{X}, \mtx{E} \rangle \le 0$. To this end, we make use of the property of the matrix $\mtx{E}$, which can be decomposed as in~\eqref{def:sterms} into the block-diagonal part (i.e., the term within inlier clusters $ \mtx{\Psi} $), the off-diagonal part (i.e., the term between inlier clusters $ \mtx{\Phi} $), the outlier-inlier part (i.e., the term between inliers and outliers $ \mtx{\Gamma} $), and the within-outlier part (i.e., the term within outlier set $ \mtx{\Lambda} $).  For example, the element in the block-diagonal part is the sum of some inliers' degree terms and a Bernoulli random variable with a relatively large parameter, while the element in the outlier-inlier part is the sum of some outliers' degree terms. As mentioned, in Step 1 we establish several structural properties of $\mtx{X^{\star}}$.  Combining these properties of $\mtx{X^{\star}}$ and those of $\mtx{E}$,  we can apply  probability concentration inequalities to separately bound the four terms $ \langle \mtx{X}^*-\mtx{X}, \mtx{\Psi} \rangle $, $ \langle \mtx{X}^*-\mtx{X}, \mtx{\Phi} \rangle $, $ \langle \mtx{X}^*-\mtx{X}, \mtx{\Gamma} \rangle $ and $ \langle \mtx{X}^*-\mtx{X}, \mtx{\Lambda} \rangle $ that contribute to $\Delta \left(\mtx{X}\right)$. 

The most challenging point lies in proving that matrix $ \mtx{\Lambda} $ corresponding to the outliers is positive semidefinite. To achieve this, we need to choose the tuning parameter $\alpha$ appropriately, and relate the matrix $ \mtx{\Lambda} $ to another matrix $ \mtx{\tilde{\Lambda}} $, which excludes the ``between inliers'' matrix $ \mtx{\Phi} $.
Then the problem becomes proving that $ \mtx{\tilde{\Lambda}} $ is a positive semidefinite  matrix. We again separate $ \mtx{\tilde{\Lambda}} $ in the inlier part and outlier part and prove the positive semidefinite  matrix property by Gershgorin Theorem~\citep{horn2012}, namely that the absolute value of the diagonal entry is larger than the sum of all off-diagonal entries in the same row.
Another difficulty is that we need to adjust the parameter $ d_i $ so that it has appropriate lower and upper bounds when we apply the Gershgorin Theorem. This is the technical reason why we use $ d_i^* $ instead of $d_i$ in the diagonal penalization term in~\eqref{eq:cvxnew}.\\

Before proceeding with the proof, we note several useful facts. The condition~\eqref{eqn: delta} in Theorem~\ref{thm:exact} implies that $ \theta_i H_a \ge m $, i.e.,  an inlier's expected number of connections to other inliers is larger than the number of outliers.
Moreover, we also have the following upper bound on the maximum of degree of an inlier: with probability at least $ 1-\frac{1}{n^2} $,
\begin{equation}
d_i \le 4\theta_i H^+.
\end{equation}
This bound can be proved using the Chernoff's inequality, which ensures that   $ d_i \le \left(1+\frac{\delta}{5B_{aa}}\right) f_i + m \le 2(\theta_i H_a + m) \le 4\theta_i H_a$ with probability at least $ 1-\frac{1}{n^3} $. 
Finally, we have the relationship  $ \bar{\theta}\ge \bar{\theta}_{\min} > C_0 > 0 $, which follows from the definitions of these quantities and the condition~\eqref{eqn: delta}.

\subsection{Step 1: Solution Candidate}
In this section, we construct a candidate solution $ \mtx{X} $ feasible to our convex relaxation~(\ref{eq:cvxnew}). Define the matrices 
\begin{align}
		\mtx{\tilde{W}} &:= \alpha \diag\left(\mtx{d^{*}_{(r+1)}}\right) + \lambda \mtx{d_{(r+1)}}\mtx{d_{(r+1)}^T} - \mtx{W}, \label{eq:tildeW}\\
		\mtx{\tilde{Z}_a} &:= \lambda \mtx{d_{(a)} d_{(r+1)}^T} - \mtx{Z_a}, \quad a = 1, \cdots, r.
\end{align}
Consequently, we have the expression
\begin{equation}
\begin{aligned}
\mtx{E} &:= \alpha \diag\left\{\mtx{d^{*}}\right\} + \lambda \mtx{d}\mtx{d^\transp} - \mtx{A} \\ 
&=\begin{bmatrix}
\alpha \diag{\left(\mtx{d^{*}_{(1)}}\right)} + \lambda\mtx{d_{(1)} d_{(1)}^T}-\mtx{K_{11}} & \cdots & \lambda \mtx{d}_{(1)} \mtx{d}_{(r)}^T - \mtx{K_{1r}} & \mtx{\tilde{Z_1}} \\
\vdots & \ddots & \vdots & \vdots \\
\lambda \mtx{d}_{(r)} \mtx{d}_{(1)}^T - \mtx{K_{1r}}^T & \cdots & \alpha \diag{\left(\mtx{d^{*}_{(r)}}\right)} + \lambda\mtx{d_{(r)} d_{(r)}^T}-\mtx{K_{rr}} & \mtx{\tilde{Z_r}} \\
\mtx{\tilde{Z_1}^T} & \cdots & 	\mtx{\tilde{Z_r}^T} & \mtx{\tilde{W}}\\
\end{bmatrix}.\\
\end{aligned}
\end{equation}

Since the desired candidate solution of optimization problem (\ref{eq:cvxnew}) has a block-diagonal structure in the inlier part, the cost of inlier part is fixed. We therefore focus on minimizing the cost of the outlier part. The objective function of the optimization problem (\ref{eq:dual}) is actually the $ n+1, n+2, \cdots, n+m $ rows and columns of the objective function of (\ref{eq:cvxnew}). The following lemma, proved in Appendix~\ref{sec:proof_lemma:1},  guarantees the existence of $ r $ vectors $ \mtx{x}_1, \cdots, \mtx{x}_r \in \mathbb{R}^m$. These vectors are used to construct a candidate solution.

\begin{lemma} \label{lemma:1}
	If assumptions (\ref{eqn:lambda}) and (\ref{eqn:alpha}) hold, then the solution to \begin{equation} 	\label{eq:dual}
	\begin{aligned}
	&\min && \sum\limits_{a=1}^{r}\langle \mtx{x}_a, \mtx{\tilde{Z}}_a^T \mtx{1}_{l_a} \rangle + \frac{1}{2} \sum\limits_{a=1}^{r} \mtx{x}_a^T \mtx{\tilde{W}} \mtx{x}_a	\\
	&\text{subject to} && \mtx{x}_a \ge \mtx{0} \quad \text{for } 1\le a \le r,	\\
	&~              && \sum\limits_{a=1}^{r} \mtx{x}_a^T \left(\mtx{e}_a \mtx{e}_j^T\right) \mtx{x}_a \le 1 \quad \text{for } 1\le j \le m,
	\end{aligned}
	\end{equation}
	exists and is unique. 
	Moreover, denote the solutions by $ \mtx{x}_1, \cdots, \mtx{x}_r \in \mathbb{R}^m $, which by definition satisfy $ \norm{x_a}_{\infty} \le 1 $. Then there are nonnegative vectors $ \mtx{\beta}_1, \cdots, \mtx{\beta}_r \in \mathbb{R}^m $ and an $ m \times m $ nonnegative diagonal matrix \begin{equation}
		\mtx{\Xi} = \diag\left\{ \xi_1, \cdots, \xi_r \right\}
	\end{equation}
	such that 
	\begin{align}
		\mtx{\tilde{W} x_a} + \mtx{\tilde{Z}_a^T 1_{l_a}} &= \mtx{\beta_a - \Xi x_a}\label{eq:KKT1}, \\
		\xi_j \left( 1 - \sum\limits_{a=1}^{r} \mtx{x}_a^T \left(\mtx{e}_a \mtx{e}_j^T\right) \mtx{x}_a \right) &= 0 \quad \text{for } 1\le j \le m \label{eq:KKT2}, \\ 
		\langle \mtx{x}_a, \mtx{\beta}_a \rangle &= 0 \quad \text{for } 1\le a \le r \label{eq:KKT3}. 		
	\end{align}
	In addition, we have \begin{equation}
		\mtx{x}_a^T \left( \mtx{\tilde{W} + \Xi} \right)\mtx{x}_b \le m\sqrt{l_a l_b} \quad \text{for } 1\le a, b \le r.
	\end{equation}
	Furthermore, for all $ a = 1, \cdots, r $ and $ j = 1, \cdots, m $, we have \begin{equation}
		\beta_{a_j} + \mtx{e}_j^T \mtx{Z}_a^T\mtx{1}_{l_a} \le \alpha d^{*}_{n+j} x_{a_j} + \lambda d_{n+j} \sum\limits_{s=1}^{m} d_{n+s}x_{a_s}  + \xi_j x_{a_j} + \lambda \tilde{d}_a d_{n+j}.
	\end{equation}
	Finally, for all $ a = 1, \cdots, r $, we have \begin{equation}
		\mtx{0} \le \mtx{\beta_a} \le \lambda \left(\tilde{d}_a + \tilde{d}_{r+1}\right) \mtx{d_{(r+1)}}.
	\end{equation}
\end{lemma}

To proceed, we define the matrices 
\begin{equation}
\mtx{V}^* = \left[\mtx{v}_1^*, \cdots, \mtx{v}_r^* \right] = \begin{bmatrix}
\mtx{1}_{l_1} &\mtx{0}&\cdots&\mtx{0}\\
\mtx{0}&\mtx{1}_{l_2}&\cdots&\mtx{0}\\
\vdots&\vdots&\ddots&\vdots\\
\mtx{0}&\mtx{0}&\cdots&\mtx{1}_{l_r}\\
\mtx{x}_1&\mtx{x}_2&\cdots&\mtx{x}_r\\
\end{bmatrix}
\end{equation}
and \begin{equation}
\mtx{X}^* = \mtx{V}^* \mtx{V}^{*T} = \begin{bmatrix}
\mtx{J_{l_1}} &\cdots&\mtx{0}&\mtx{1}_{l_1}\mtx{x}_1^T\\
\vdots&\ddots&\vdots&\vdots\\
\mtx{0}&\cdots&\mtx{J_{l_r}}&\mtx{1}_{l_r}\mtx{x}_r^T\\
\mtx{x}_1\mtx{1}_{l_1}^T&\cdots&\mtx{x}_r\mtx{1}_{l_r}^T&\sum\limits_{a=1}^{r} \mtx{x}_a \mtx{x}_a^T
\end{bmatrix}
\end{equation}
Since $ \mtx{x_a} $'s are feasible to optimization problem~(\ref{eq:dual}), we can easily see that $ \mtx{X}^* $ is feasible to optimization problem~(\ref{eq:cvxnew}). In the sequel, we will prove that the $ \mtx{X}^* $ is actually an optimal solution to (\ref{eq:cvxnew}). 

%

\subsection{Step 2:  Verification of the solution to the dual problem}

To establish the theorem, it suffices to show for any feasible solution $ \mtx{X} $ to the program~\eqref{eq:cvxnew} with $ \mtx{X} \neq \mtx{X}^* $, there holds 
\begin{equation}
	\Delta \left(\mtx{X}\right) \triangleq \langle \mtx{X}^*-\mtx{X}, \alpha \diag\left\{\mtx{d^{*}}\right\} + \lambda \mtx{d}\mtx{d^\transp} - \mtx{A} \rangle 
	< 0.
\end{equation}

To this end, we will prove that $\Delta \left(\mtx{X}\right) $ can be decomposed as 
\begin{equation}
\begin{aligned}
\Delta(\mtx{X}) &= \langle \mtx{X}^*-\mtx{X}, \alpha \diag\left\{\mtx{d^{*}}\right\} + \lambda \mtx{d}\mtx{d^\transp} - \mtx{A}\rangle\\
				&= \langle \mtx{X}^*-\mtx{X}, \mtx{\Psi} + \mtx{\Phi} +\mtx{\Gamma}+ \mtx{\Lambda} \rangle \label{def:sterms}\\
				&=: S_1 + S_2 + S_3 + S_4,
\end{aligned}
\end{equation}
where the matrices $\mtx{\Psi}$, $\mtx{\Phi}$ and $\mtx{\Gamma}$
have the  form
\begin{equation*}
	\mtx{\Psi} = \begin{bmatrix}
	-\mtx{\Psi}_{11}&\cdots&\mtx{0}&\mtx{0}\\
	\vdots&\ddots&\vdots&\vdots\\
	\mtx{0}&\cdots&-\mtx{\Psi}_{rr}&\mtx{0}\\
	\mtx{0}&\cdots&\mtx{0}&\mtx{0}
	\end{bmatrix}, \quad
	\mtx{\Phi} = \begin{bmatrix}
	\mtx{0}&\cdots&\mtx{\Phi}_{1r}&\mtx{0}\\
	\vdots&\ddots&\vdots&\vdots\\
	\mtx{\Phi}_{1r}^T&\cdots&\mtx{0}&\mtx{0}\\
	\mtx{0}&\cdots&\mtx{0}&\mtx{0}
	\end{bmatrix},
\end{equation*}
\begin{equation*}
	\mtx{\Gamma} = \begin{bmatrix}
	\mtx{0}&\cdots&\mtx{0}&\frac{1}{G_1}\mtx{\theta}_{(1)}\mtx{\beta}_1^T \\
	\vdots&\ddots&\vdots&\vdots\\
	\mtx{0}&\cdots&\mtx{0}&\frac{1}{G_r}\mtx{\theta}_{(r)}\mtx{\beta}_r^T\\
	\frac{1}{G_1}\mtx{\beta}_1\mtx{\theta}_{(1)}^T&\cdots&\frac{1}{G_r}\mtx{\beta}_r\mtx{\theta}_{(r)}^T&-\mtx{\Xi}
	\end{bmatrix}
\end{equation*}
and the matrix  $ \mtx{\Lambda} = \alpha \diag\left\{\mtx{d^{*}}\right\} + \lambda \mtx{d}\mtx{d^\transp} - \mtx{A} - (\mtx{\Psi} + \mtx{\Phi} +\mtx{\Gamma})$ satisfies $ \mtx{\Lambda} \mtx{V^{\star}} = \mtx{0} $.

\bigskip
In the following, we will construct one by one the matrices  $ \mtx{\Lambda}, \mtx{\Psi}$ and $ \mtx{\Phi} $ in the decomposition \eqref{def:sterms} and prove that $ \mtx{\Psi}_{aa} > 0 $, $ \mtx{\Phi}_{ab} > 0 $ and $ \mtx{\Lambda} \succeq \mtx{0} $. Finally, we will prove that $S_1 <0$ and $S_i \le 0 $ for $i=2,3,4 $, from which we can conclude that $\Delta \left(\mtx{X}\right) < 0$ and thereby finish the proof.

\subsubsection{Construction of $ \Psi_{aa} $ and $ \Phi_{ab} $ in  \eqref{def:sterms}}
\label{sec:psiaa}

The equality $ \mtx{\Lambda} \mtx{V^{*}} = \mtx{0} $ yields that
\begin{numcases}{}
\left(\mtx{\tilde{Z_a}}^T -\frac{1}{G_a} \mtx{\beta_a}\mtx{\theta_{(a)}}^T \right)\mtx{1}_{l_a} + \left(\tilde{\mtx{W}}+\mtx{\Xi}\right)\mtx{x}_a = \mtx{0}\label{eq:existence1} \\
\left( \alpha \diag \left\{\mtx{d_{(a)}^{*} } \right\} + \lambda \mtx{d_{(a)}} \mtx{d_{(a)}^T}- \mtx{K_{aa}}+\mtx{\Psi_{aa}}\right) \mtx{1_{l_a}} + \left(\mtx{\tilde{Z_a}} -\frac{1}{G_a} \mtx{\theta_{(a)}}\mtx{\beta_a^T} \right)\mtx{x_a} = \mtx{0}\label{eq:existence2}\\
\left( \lambda \mtx{d_{(a)}} \mtx{d_{(b)}^T}- \mtx{K_{ab}}-\mtx{\Phi_{ab}}\right) \mtx{1_{l_b}} + \left(\mtx{\tilde{Z_a}} -\frac{1}{G_a} \mtx{\theta_{(a)}}\mtx{\beta_a^T} \right)\mtx{x_b} = \mtx{0}\label{eq:existence3}\\
\left( \lambda \mtx{d_{(b)}} \mtx{d_{(a)}^T}- \mtx{K_{ab}}^T-\mtx{\Phi_{ab}}^T\right) \mtx{1_{l_a}} + \left(\mtx{\tilde{Z_b}} -\frac{1}{G_b} \mtx{\theta_{(b)}}\mtx{\beta_b^T} \right)\mtx{x_a} = \mtx{0}\label{eq:existence4}
\end{numcases}
It is clear that (\ref{eq:existence1}) is equivalent to (\ref{eq:KKT1}). In the following, we will construct $ \mtx{\Psi}_{aa} $ satisfying (\ref{eq:existence2}) and  $ \mtx{\Phi}_{ab} $ satisfying both (\ref{eq:existence3}) and (\ref{eq:existence4}). 

The equality (\ref{eq:existence2}) is equivalent to \begin{equation}
	\begin{aligned}
	\mtx{\Psi}_{aa} \mtx{1}_{l_a} &= \left( \mtx{K_{aa}} - \alpha \diag \left\{\mtx{d_{(a)}^{*} } \right\} - \lambda \mtx{d_{(a)}} \mtx{d_{(a)}^T} \right) \mtx{1}_{l_a} - \left( \lambda \mtx{d_{(a)}} \mtx{d_{(r+1)}^T} - \mtx{Z_a} - \frac{1}{G_a} \mtx{\theta_{(a)}}\mtx{\beta_a^T} \right) \mtx{x}_a \\
	&= \mtx{K_{aa}} \mtx{1}_{l_a} - \alpha \mtx{d}_{(a)}^* - \lambda \tilde{d}_a \mtx{d}_{(a)} + \mtx{Z}_a \mtx{x}_a - \lambda\mtx{d_{(a)}} \mtx{d_{(r+1)}^T} \mtx{x}_a,
	\end{aligned}
\end{equation}
where the last equality is due to $ \langle \mtx{x}_a, \mtx{\beta}_a \rangle = 0 $. To ensure $ \mtx{\Psi}_{aa} > 0 $, we  construct $ \mtx{\Psi}_{aa} $ as the sum of a non-negative diagonal matrix plus a positive matrix. In particular, we set
\begin{equation}
\begin{aligned}
\mtx{\Psi_{aa}}  \triangleq \ & \diag \left\{ \mtx{K_{aa}}\mtx{1_{l_a}} + \mtx{Z_a}\mtx{x_a} \right\} - \left[ \lambda \left( \mtx{d_{(r+1)}^T} \mtx{x_a} \right) + \lambda \tilde{d}_a  \right]\diag\left\{\mtx{d_{(a)}}\right\}\\
&-\alpha \diag \left\{\mtx{d_{(a)}^{*} }\right\}- \epsilon G_a \diag\left\{\mtx{\theta_{(a)}}\right\} + \epsilon \mtx{\theta_{(a)}}  \mtx{\theta_{(a)}^T}.
\end{aligned}
\end{equation} 
Setting $ \epsilon = \frac{\delta}{10} $ satisfies our requirements.

Next let us construct $ \mtx{\Phi}_{ab} \in \mathbb{R}^{l_a \times l_b} $ satisfying both (\ref{eq:existence3}) and (\ref{eq:existence4}). These two equalities are equivalent to
\begin{numcases}
	\mtx{\Phi}_{ab} \mtx{1_{l_b}} - \lambda \tilde{d_b} \mtx{d_{(a)}} = - \mtx{K_{ab}}\mtx{1_{l_b}} + \mtx{\tilde{Z_a}} \mtx{x_b} - \frac{1}{G_a} \mtx{\theta_{(a)}} \mtx{\beta_a^T} \mtx{x_b} \triangleq \mtx{a}, \label{eq:phiab1}\\
	\mtx{\Phi}_{ab}^T \mtx{1_{l_a}} - \lambda \tilde{d_a} \mtx{d_{(b)}} = - \mtx{K_{ab}^T}\mtx{1_{l_a}} + \mtx{\tilde{Z_b}} \mtx{x_a} - \frac{1}{G_b} \mtx{\theta_{(b)}} \mtx{\beta_b^T} \mtx{x_a} \triangleq \mtx{b}. \label{eq:phiab2}
\end{numcases}
One can verify that 
\begin{equation}
\mtx{1_{l_a}^T} \mtx{a} = - \mtx{1_{l_a}^T} \mtx{K_{ab} 1_{l_b}} - \mtx{x_a^T \left(\tilde{W}+\Xi\right)x_b}= \mtx{1_{l_b}^T b} \triangleq s.
\end{equation}
If we set
\begin{equation}
\mtx{\Phi_{ab}} \triangleq \frac{1}{G_b}\mtx{a \theta_{(b)}^T} + \frac{1}{G_a} \mtx{\theta_{(a)}}\mtx{b^T}-\frac{s}{G_a G_b} \mtx{\theta_{(a)} \theta_{(b)}^T} + \lambda \mtx{d_{(a)} d_{(b)}^T},
\end{equation}
then $ \mtx{\Phi_{ab}} $ satisfies (\ref{eq:phiab1}) and (\ref{eq:phiab2}).
After simplification, we obtain 
\begin{equation}
\begin{aligned} \label{eq:phiabdef}
\mtx{\Phi_{ab}} = &-\left(\frac{1}{G_b}\mtx{K_{ab} 1_{l_b} \theta_{(b)}^T} + \frac{1}{G_a} \mtx{\theta_{(a)} 1_{l_a}^T K_{ab}}\right) + \frac{\mtx{1_{l_a}^T K_{ab} 1_{l_b} }}{G_a G_b}\mtx{\theta_{(a)} \theta_{(b)}^T} +  \lambda \mtx{d_{(a)} d_{(b)}^T} \\
& - \frac{1}{G_a G_b} \left(\mtx{1_{l_a}^T \tilde{Z_a} x_b} + \mtx{1_{l_b}^T \tilde{Z_b} x_a} + \mtx{x_a^T \left(\tilde{W}+\Xi\right)x_b}\right)\mtx{\theta_{(a)} \theta_{(b)}^T}\\ 
&+ \left(\frac{1}{G_a} \mtx{\theta_{(a)} x_a^T \tilde{Z_b}^T}+ \frac{1}{G_b} \mtx{\tilde{Z_a}x_b \theta_{(b)}^T}
\right).\\
\end{aligned}
\end{equation}

As we have shown, $ \mtx{\Psi} $ and $ \mtx{\Phi} $ are well defined, so $ \mtx{\Lambda} $ is given by $ \mtx{\Lambda} = \mtx{E-\Psi-\Phi-\Gamma} $. In the following, we will study the properties of these matrices and give lower bounds for terms  $ S_1, S_2, S_3$ and $ S_4 $ defined in \eqref{def:sterms}.

\subsubsection{The $ S_1 $ Term in \eqref{def:sterms}}
We will show  that 
\begin{equation}  \label{eq:Psi_aa_definition}
\mtx{\Psi_{aa}} - \epsilon \mtx{\theta_{(a)}}  \mtx{\theta_{(a)}^T} \geq \text{and} \succeq \mtx{0}.
\end{equation}
Notice that $ \mtx{\Psi_{aa}} - \epsilon \mtx{\theta_{(a)}} \mtx{\theta_{(a)}}^T $ is a diagonal matrix, so we only need to check that each entry on the diagonal is larger or equal than 0. Since $ \mtx{Z_a} \ge \mtx{0} $, $ \mtx{x_a} \ge \mtx{0}$, and $ \norm{\mtx{x_a}}_\infty \le 1$, it is sufficient to prove
\begin{equation}
\label{eqn: 1}
\sum\limits_{j \in C_a^*} K_{ij} - \alpha d_i^* - \lambda \tilde{d}_{r+1} d_i -  \lambda \tilde{d}_{a} d_i - \epsilon G_a \theta_i \ge 0 \quad \forall i \in C_a^*.
\end{equation}

Notice that \begin{equation}
	(B_{aa} - \delta)\left(1+\frac{\delta}{5p^+}\right)^2 \le (p^+ - \delta)\left( 1+ \frac{2\delta}{5p^+} + \frac{\delta^2}{25{p^+}^2}\right) = p^+ - \frac{3}{5}\delta -\frac{9\delta^2}{25p^+} - \frac{\delta^3}{25{p^+}^2} \le p^+ - \frac{3}{5}\delta
\end{equation}
and \begin{equation}
	(B_{aa} - \delta)\left(1+\frac{\delta}{5p^+}\right) \le (p^+ - \delta)\left(1+\frac{\delta}{5p^+}\right) = p^+ - \frac{4}{5}\delta -\frac{\delta^2}{5p^+} \le p^+ - \frac{4}{5}\delta.
\end{equation}
By Lemma \ref{lemma2}, we have with probability at least $ 1-\frac{1}{n^2} $ 
\begin{equation}
	\lambda \tilde{d}_a d_i \le \frac{1}{H_a^2}\left(B_{aa}-\frac{3}{5}\delta\right)\left(\sum\limits_{j\in C_a^*}f_j\right)f_i \le \frac{1}{H_a^2}\left(B_{aa}-\frac{3}{5}\delta\right)\left(G_aH_a+ml_a\right)(\theta_i H_a + m)
\end{equation}
and \begin{equation}
	\lambda \tilde{d}_{r+1} d_i \le \frac{1}{H_a^2}\left(B_{aa} - \frac{4}{5} \delta \right) f_i \tilde{d}_{r+1} \le \frac{1}{H_a^2}\left(B_{aa} - \frac{4}{5} \delta \right) (\theta_i H_a + m)m(m+n).
\end{equation}
In addition, by Chernoff's Inequality, with probability at least $ 1-\frac{1}{n^3} $ we have \begin{equation}
	\sum\limits_{j\in C_a^*} K_{ij} \ge \theta_i B_{aa}G_a - \sqrt{6\theta_i \log n B_{aa}G_a }.
\end{equation}
Thus, with probability $ \ge 1-\frac{1}{n^2} $, we have \begin{equation}
\begin{aligned}
&\frac{1}{\theta_i}\times \text{LHS of (\ref{eqn: 1})} \\ 
&\ge B_{aa}G_a- \sqrt{\frac{6\log n B_{aa}G_a}{\theta_i} } -\frac{\alpha}{\theta_i}d_i^* - \frac{\lambda}{\theta_i} \tilde{d}_{r+1}d_i - \left(B_{aa} -\frac{3}{5}\delta\right)\left(G_a+\frac{m l_a}{H_a}\right)\left(1+\frac{m}{\theta_i H_a}\right)-\epsilon G_a\\
&\ge \frac{1}{2}\delta G_a - \sqrt{\frac{6\log n B_{aa}G_a}{\theta_i} } -  \frac{\alpha}{\theta_i}d_i^* - \frac{1}{\theta_i H_a^2}\left(B_{aa}-\frac{4}{5}\delta \right)m(m+n) (\theta_i H_a + m)\\
&-B_{aa}\left(\frac{m l_a}{H_a}+\frac{mG_a}{\theta_i H_a}+\frac{m^2 l_a}{\theta_iH_a^2}\right)\\
&\ge \frac{1}{2}\delta G_a - \sqrt{\frac{6\log n B_{aa}G_a}{\theta_i} }-  \frac{\alpha}{\theta_i}d_i^* - \frac{2B_{aa}}{H_a}m(m+n)- B_{aa}\left(\frac{2m l_a}{H_a}+\frac{mG_a}{\theta_i H_a}\right),
\end{aligned}
\end{equation}
where last inequality is due to the fact that $ \theta_iH_a \ge m $.

Combining pieces, we see that the following is sufficient for our goal:
\begin{align}
&\delta \ge C \sqrt{\frac{ B_{aa}\log n }{\theta_i G_a}} \Leftarrow \delta \ge \sqrt{\frac{\log n \cdot p^+}{\theta_{\min} G_{\min}}} \label{eq:delta_cond1}\\
&\delta \ge\frac{C}{G_a} \frac{\alpha}{\theta_i}d_i^* \Leftarrow \delta \ge\frac{C}{G_a} \frac{\alpha}{\theta_i} \theta_i H^+ \Leftarrow \delta \ge C \frac{\alpha n\bar{\theta}p^+}{G_{\min}}\\
&\delta \ge C \frac{B_{aa}m(m+n)}{H_aG_a}\Leftarrow \delta \ge C \frac{p^+ m(m+n)}{n\bar{\theta}q^- G_{\min} } \Leftarrow \delta \ge C \frac{p^+}{q^-} \frac{1}{\bar{\theta}} \frac{m}{G_{\min}} \Leftarrow \delta \ge C\frac{m}{\bar{\theta}G_{\min}} \\
&\delta \ge C\frac{m l_a B_{aa}}{H_a G_a} \Leftarrow \delta \ge C \frac{m l_a p^+}{n\bar{\theta}q^- l_a \bar{\theta}_{\min}} \Leftarrow \delta \ge C \frac{m}{n\bar{\theta}\bar{\theta}_{\min}}\\
&\delta \ge \frac{m G_a B_{aa}}{\theta_i H_a G_a} \Leftarrow \delta \ge C\frac{m p^+}{n\bar{\theta}q^- \theta_{\min}} \Leftarrow \delta \ge C\frac{m }{n\bar{\theta}\theta_{\min}} .
\end{align}
Note that the condition~(\ref{eqn: delta}) in Theorem~\ref{thm:exact} fulfills all the requirements above, thus we have $\mtx{\Psi_{aa}} - \epsilon \mtx{\theta_{(a)}}  \mtx{\theta_{(a)}^T} \geq \text{and} \succeq \mtx{0}$. This implies the weaker result that $ \mtx{\Psi_{aa}} > \mtx{0} $.

Finally, we have
\begin{equation*}
\begin{aligned}
S_1 &= \langle \mtx{X}^*-\mtx{X}, \mtx{\Psi}\rangle\\
&= -\sum_{a=1}^{r} \left(\mtx{X}^*-\mtx{X}\right)_{(a,a)} \mtx{\Psi_{aa}}
< 0
\end{aligned}
\end{equation*}
where the last inequality is due to the fact that all entries of $\mtx{X^{*}}_{(a,a)} $ equal to 1 and all entries of $\mtx{X}_{(a,a)}$ are no larger than 1.

\subsubsection{The $ S_2 $ Term in  \eqref{def:sterms}}

We will first  prove that $ \mtx{\Phi_{ab} > \mtx{0}} $. For the first three terms in (\ref{eq:phiabdef}), we apply Lemma \ref{lemma3} to get
\begin{equation}
\begin{aligned} 
&-\left(\frac{1}{G_b}\mtx{K_{ab} 1_{l_b} \theta_{(b)}^T} + \frac{1}{G_a} \mtx{\theta_{(a)} 1_{l_a}^T K_{ab}}\right) + \frac{\mtx{1_{l_a}^T K_{ab} 1_{l_b} }}{G_a G_b}\mtx{\theta_{(a)} \theta_{(b)}^T} +  \lambda \mtx{d_{(a)} d_{(b)}^T} \\
\ge &-2\left(B_{ab}+\frac{1}{20}\delta\right) \mtx{\theta_{(a)}\theta_{(b)}^T}+\left(B_{ab}-\frac{1}{25}\delta\right)\mtx{\theta_{(a)}\theta_{(b)}^T}+\left(B_{ab}+\frac{6}{25}\delta \right) \mtx{\theta_{(a)} \theta_{(b)}^T}\\
=& \frac{\delta}{10} \mtx{\theta_{(a)} \theta_{(b)}^T},\\
\end{aligned}
\end{equation}
Lemma \ref{lemma:1} also proves that $ \mtx{x_a^T \left(\tilde{W} + \Xi \right) x_b} \le m \sqrt{l_a l_b} $.

To bound the forth and fifth term in (\ref{eq:phiabdef}), we  first bound $ \mtx{\tilde{Z_a} x_b} = \left(  \lambda \mtx{d_{(a)} d_{(r+1)}^T } - \mtx{Z_a} \right) \mtx{x_b}$. Since $ \mtx{x_b} \ge 0 $, $ \norm{\mtx{x_b}}_\infty < 1 $ and $ \mtx{Z_a} $ is a 0-1 matrix, we have 
\begin{equation} \label{eq: zaxb}
\begin{aligned}
\mtx{\tilde{Z_a} x_b} &\le \lambda \mtx{d_{(a)} d_{(r+1)}^T 1_m} \le \frac{B_{aa}-\delta}{H_a^2} \left( 1+\frac{\delta}{5B_{aa}} \right) \mtx{f_{(a)}} \tilde{d}_{r+1} \\
&\le \frac{2}{H_a}\left(B_{aa} - \frac{4}{5}\delta\right)\tilde{d}_{r+1} \mtx{\theta_{(a)}} \le \frac{2m(m+n)}{H_a}\left(B_{aa} - \frac{4}{5}\delta\right) \mtx{\theta_{(a)}},\\
\mtx{\tilde{Z_a} x_b} &\ge - \mtx{Z_a x_b} \ge -\mtx{J_{(l_a, m)} 1_m}  = -m\mtx{1_{l_a}},
\end{aligned}
\end{equation}
where upper bound of $ \mtx{\tilde{Z_a} x_b} $ is due to the facts that $ f_i = \theta_i H_a + m \le 2\theta_i H_a $ and $ \left(B_{aa}-\delta\right)\left( 1+\frac{\delta}{5B_{aa}} \right) \le B_{aa} -\frac{4}{5}\delta$. 

Therefore, to prove $ \mtx{\Phi_{ab}} > 0 $, we only need to prove that 
\begin{equation}
\begin{aligned}
\frac{\delta}{10} \mtx{\theta_{(a)} \theta_{(b)}} > & \frac{1}{G_a G_b} \left( G_a \frac{2m(m+n)}{H_a}\left(B_{aa} - \frac{4}{5}\delta\right) + G_b\frac{2m(m+n)}{H_b}\left(B_{bb} - \frac{4}{5}\delta\right) \right) \mtx{\theta_{(a)} \theta_{(b)}} \\
&+ \frac{m\sqrt{l_a l_b}}{G_a G_b}\mtx{\theta_{(a)} \theta_{(b)}} + \left( \frac{m}{G_a} \mtx{ \theta_{(a)}^T 1_{l_b}} + \frac{m}{G_b} \mtx{1_{l_a} \theta_{(b)}^T} \right),
\end{aligned}
\end{equation}
which is implied by
\begin{eqnarray}
\begin{aligned}
&\frac{\delta}{50} \ge \frac{2m(m+n)}{G_b H_a} \left(B_{aa} - \frac{4}{5}\delta\right) \Leftarrow \delta \ge C \frac{B_{aa}m(m+n)}{n\bar{\theta}q^- G_{\min} } \Leftarrow \delta \ge C \frac{p^+}{q^-} \frac{1}{\bar{\theta}} \frac{m}{G_{\min}} \Leftarrow \delta \ge \frac{m}{G_{\min}}\\
&\frac{\delta}{50} \ge \frac{2m(m+n)}{G_a H_b} \left(B_{bb} - \frac{4}{5}\delta\right)\Leftarrow \delta \ge C \frac{B_{bb}m(m+n)}{n\bar{\theta}q^- G_{\min} } \Leftarrow \delta \ge C \frac{p^+}{q^-} \frac{1}{\bar{\theta}} \frac{m}{G_{\min}} \Leftarrow \delta \ge \frac{m}{G_{\min}} \\
& \frac{\delta}{50} \ge \frac{m\sqrt{l_a l_b}}{G_a G_b} \Leftarrow \delta \ge C\frac{m}{\theta_{\min} G_{\min}}\\
& \frac{\delta}{50} \theta_j \ge \frac{m}{G_a} \Leftarrow \delta \ge C\frac{m}{\theta_{\min} G_{\min}} \\
& \frac{\delta}{50} \theta_i \ge \frac{m}{G_b} \Leftarrow \delta \ge C\frac{m}{\theta_{\min} G_{\min}}.
\end{aligned}
\end{eqnarray}
Note that the condition~(\ref{eqn: delta}) in Theorem~\ref{thm:exact} fulfills all the requirements above, thus we have $ \mtx{\Phi_{ab}} > 0 $.

Finally, we have \begin{equation}
\begin{aligned}
S_2 &= \langle \mtx{X}^*-\mtx{X}, \mtx{\Phi}\rangle\\
&= \sum_{a\neq b} \left(\mtx{X}^*-\mtx{X}\right)_{(a,b)} \mtx{\Phi_{ab}}
\le 0,
\end{aligned}
\end{equation}
where the last inequality is due to the fact that $\mtx{X^{*}}_{(a,b)} = \mtx{0}$ and $\mtx{X}_{(a,b)} \ge \mtx{0}$.

\subsubsection{The $ S_3 $ Term in \eqref{def:sterms}}

By the feasibility of $\mtx{X}$ and the non-negativity of $\mtx{\beta}_a$ and $\mtx{\Xi}$, we have 
\begin{equation}
	\langle \mtx{X_{(a, r+1)}}, \  \frac{1}{G_a} \mtx{\theta_{(a)} \beta_a^T} \rangle \ge 0
\end{equation}
and \begin{equation}
\langle \mtx{X_{(r+1, r+1)}}, \  \mtx{\Xi} \rangle \ge 0.
\end{equation}
By (\ref{eq:KKT3}), i.e. $\langle \mtx{x}_a, \mtx{\beta}_a \rangle = 0$,  we have \begin{equation}
	\langle \mtx{1_{l_a} x_a^T}, \  \frac{1}{G_a} \mtx{\theta_{(a)} \beta_a^T}\rangle = 0.
\end{equation}
By (\ref{eq:KKT2}), we have \begin{equation}
\langle \mtx{J}_m - \sum_{a=1}^{r}\mtx{x_a x_a^T}, \  \mtx{\Xi} \rangle = 0.
\end{equation}
It follows that
\begin{equation*}
\begin{aligned}
S_3 &= \langle \mtx{X}^*-\mtx{X}, \  \mtx{\Gamma}\rangle\\
& = 2\sum_{a=1}^{r} \langle \mtx{1_{l_a} x_a^T}, \  \frac{1}{G_a} \mtx{\theta_{(a)} \beta_a^T}\rangle - \langle \sum_{a=1}^{r}\mtx{x_a x_a^T}, \  \mtx{\Xi}\rangle - \langle \mtx{X_{(a, r+1)}}, \  \frac{1}{G_a} \mtx{\theta_{(a)} \beta_a^T} \rangle + \langle \mtx{X_{(r+1, r+1)}}, \  \mtx{\Xi} \rangle\\
&= - \langle \mtx{X_{(a, r+1)}}, \  \frac{1}{G_a} \mtx{\theta_{(a)} \beta_a^T} \rangle - \langle \mtx{J_m} - \mtx{X_{(r+1, r+1)}}, \  \mtx{\Xi} \rangle\\
&\le 0,
\end{aligned}
\end{equation*}
where the last inequality is due to the fact that all entries of $\mtx{X_{(r+1, r+1)}}$ are no larger than 1.

\subsubsection{The $ S_4 $ Term in \eqref{def:sterms}}

We will first prove that $ \mtx{\Lambda} \succeq \mtx{0} $. The condition $ \mtx{\Lambda V^{*}} = \mtx{0} $ implies that $ \rank\left(\mtx{\Lambda}\right) \le N-r $. Thus we only need to prove that the $ (N-r) $-th largest eigenvalue of $ \mtx{\Lambda} $ is no smaller than $ 0 $ while all other smaller eigenvalues are equal to $ 0 $.

We define the matrix
\begin{equation}
\mtx{\hat{V}}=\begin{bmatrix}
\frac{1}{\sqrt{ \norm{\mtx{\theta_{(1)}}}_1 }}\mtx{\theta_{(1)}} & \mtx{0} & \cdots & \mtx{0}\\
\mtx{0}  &  \frac{1}{\sqrt{\norm{\mtx{\theta_{(2)}}}_1} } \mtx{\theta_{(2)}} & \cdots\ & \mtx{0} \\
\vdots & \vdots & \ddots & \vdots \\
\mtx{0}   & \mtx{0} & \cdots  &  \frac{1}{\sqrt{\norm{\mtx{\theta_{(r)}}}_1} } \mtx{\theta_{(r)}} \\
\mtx{0} & \mtx{0} & \cdots & \mtx{0} \\
\end{bmatrix} \in \mathbb{R}^{N\times r}
\end{equation}
One sees that $  \mtx{\hat{V}} $ is a basis matrix, i.e., the columns of $  \mtx{\hat{V}} $ are orthogonal unit vectors. Take $ \mtx{\hat{V}}_\perp \in \mathbb{R}^{N\times (N-r)} $ such that $ \mtx{U} = \left[ \mtx{\hat{V}}_\perp, \mtx{\hat{V}} \right] $ is an orthogonal matrix. 
Define the matrix 
\begin{equation}
\begin{aligned}
\mtx{\tilde{\Lambda}} &\triangleq \begin{bmatrix}
\mtx{\tilde{\Lambda}_1} & \mtx{\tilde{\Lambda}_2} \\
\mtx{\tilde{\Lambda}^T_2} & \mtx{\tilde{W}} + \mtx{\tilde{\Xi}} \\
\end{bmatrix} \\
&\triangleq \begin{bmatrix}
\tiny{\alpha \diag{\left(\mtx{d^{*}_{(1)}}\right)} + \lambda\mtx{d_{(1)} d_{(1)}^T} + B_{11} \mtx{\theta_{(1)} \theta_{(1)}^T}-\mtx{K_{11}} + \mtx{\Psi_{11}}} & \cdots & B_{1r}\mtx{\theta_{(1)} \theta_{(r)}^T} - \mtx{K_{1r}} & \mtx{\tilde{Z_1}} - \frac{1}{G_1} \mtx{\theta_{(1)}} \mtx{\beta_1^T} \\
\vdots & \ddots & \vdots & \vdots \\
B_{1r}\mtx{\theta_{(r)} \theta_{(1)}^T} - \mtx{K_{1r}^T} & \cdots & \tiny{\alpha \diag{\left(\mtx{d^{*}_{(r)}}\right)} + \lambda\mtx{d_{(r)} d_{(r)}^T} + B_{rr} \mtx{\theta_{(r)} \theta_{(r)}^T}-\mtx{K_{rr}} + \mtx{\Psi_{rr}}} & \mtx{\tilde{Z_r}} - \frac{1}{G_r} \mtx{\theta_{(r)}} \mtx{\beta_r^T} \\
\mtx{\tilde{Z_1}^T} - \frac{1}{G_1}  \mtx{\beta_1}\mtx{\theta_{(1)}^T} & \cdots & 	\mtx{\tilde{Z_r}^T} - \frac{1}{G_r}  \mtx{\beta_r}\mtx{\theta_{(r)}^T} & \mtx{\tilde{W}+\Xi}\\
\end{bmatrix}
\end{aligned}
\end{equation}
The matrix $ \mtx{\tilde{\Lambda}} $ is close to $ \mtx{\Lambda} $ in the sense that
\begin{equation}
\mtx{\tilde{\Lambda} - \Lambda} =  \begin{bmatrix}
B_{11} \mtx{\theta_{(1)} \theta_{(1)}^T} & \cdots & B_{1r}\mtx{\theta_{(1)} \theta_{(r)}^T} - \lambda \mtx{d_{(1)} d_{(r)}^T } + \mtx{\Phi_{1r}} & \mtx{0} \\

\vdots & \ddots  & \vdots & \vdots \\
B_{1r}\mtx{\theta_{(r)} \theta_{(1)}^T}- \lambda \mtx{d_{(r)} d_{(1)}^T } + \mtx{\Phi_{1r}^T} & \cdots & B_{rr}\mtx{\theta_{(r)} \theta_{(r)}^T} & \mtx{0} \\
\mtx{0} & \cdots & 	\mtx{0} & \mtx{0}\\
\end{bmatrix}
\end{equation}

Note that each entry of $ \mtx{\tilde{\Lambda} - \Lambda} $ takes the form of $ C \mtx{\theta_{(a)} \theta_{(b)}^T} $, where $ C $ is a constant. Thus we have $ \mtx{\hat{V}}_\perp^T \left(\mtx{\tilde{\Lambda} - \Lambda}\right) \mtx{\hat{V}}_\perp  = 0 $, or $ \mtx{\hat{V}}_\perp^T \mtx{\tilde{\Lambda}} \mtx{\hat{V}}_\perp  = \mtx{\hat{V}}_\perp^T \mtx{\Lambda} \mtx{\hat{V}}_\perp $. 
Since the matrix 
$$ 
\mtx{U^T \Lambda U} \triangleq \begin{bmatrix}
\mtx{\hat{V}^T_{\perp} \Lambda \hat{V}_{\perp}} & \mtx{\hat{V}^T_{\perp} \Lambda \hat{V}} \\
\mtx{\hat{V}^T \Lambda \hat{V}_{\perp}} & \mtx{\hat{V}^T \Lambda \hat{V}} \\
\end{bmatrix}  $$
has the same eigenvalues as $ \mtx{\Lambda  }$ does, Weyl's Inequality implies that
\begin{equation}
	\lambda_{N-r}(\mtx{\Lambda}) = \lambda_{N-r}(\mtx{U^T \Lambda U}) \ge \lambda_{N-r}(\mtx{\hat{V}^T_{\perp} \Lambda \hat{V}_{\perp}}) = \lambda_{N-r}(\mtx{\hat{V}^T_{\perp} \tilde{\Lambda} \hat{V}_{\perp}}) \ge \lambda_{N}(\mtx{U^T \Lambda U}) = \lambda_N(\mtx{\tilde{\Lambda}}).
\end{equation}
Thus we only need to prove $  \mtx{\tilde{\Lambda}} \succ \mtx{0}$.

To this end, we consider the decomposition
\begin{equation}
\begin{aligned}
\mtx{\tilde{\Lambda}_1}  =& \mtx{F_1 + F_2} \\
& \begin{bmatrix}
\tiny{\alpha \diag{\left(\mtx{d^{*}_{(1)}}\right)} + \lambda\mtx{d_{(1)} d_{(1)}^T} + B_{11} \mtx{\theta_{(1)} \theta_{(1)}^T}-\mtx{K_{11}} + \mtx{\Psi_{11}}} & \cdots & \mtx{0} \\
\vdots & \ddots & \vdots \\
\mtx{0} & \cdots & \tiny{\alpha \diag{\left(\mtx{d^{*}_{(r)}}\right)} + \lambda\mtx{d_{(r)} d_{(r)}^T} + B_{rr} \mtx{\theta_{(r)} \theta_{(r)}^T}-\mtx{K_{rr}} + \mtx{\Psi_{rr}}} 
\end{bmatrix}\\
& + \begin{bmatrix}
\mtx{0} & \cdots & B_{1r}\mtx{\theta_{(1)} \theta_{(r)}^T} - \mtx{K_{1r}} \\
\vdots & \ddots & \vdots \\
B_{1r}\mtx{\theta_{(r)} \theta_{(1)}^T} - \mtx{K_{1r}^T} & \cdots & \mtx{0}
\end{bmatrix}.
\end{aligned}
\end{equation}
In Section~\ref{sec:psiaa}, we proved $ \mtx{\Psi_{aa}} \succeq \frac{\delta}{4} G_a \diag\left( \mtx{\theta_{(a)}} \right) + \frac{\delta}{10} \mtx{\theta_{(a)} \theta_{(a)}^T} $. Combining with (\ref{eqn: norm}), we have 
\begin{equation*}
\begin{aligned}
&\alpha \diag{\left(\mtx{d^{*}_{(a)}}\right)} + \lambda\mtx{d_{(a)} d_{(a)}^T} + \left(B_{aa} \mtx{\theta_{(a)} \theta_{(a)}^T}-\mtx{K_{aa}}\right) + \mtx{\Psi_{aa}} \\
\succeq &\alpha \diag{\left(\mtx{d_{(a)}}\right)} + \lambda\mtx{d_{(a)} d_{(a)}^T} + \left(B_{aa} \mtx{\theta_{(a)} \theta_{(a)}^T}-\mtx{K_{aa}}\right) + \mtx{\Psi_{aa}} \\
\succeq & \alpha \left(1-\frac{\delta}{5B_{aa}}\right) \diag\left(\mtx{f_{(a)}}\right) + \mtx{0} - \left(2\log l_a + \sqrt{6 \theta_{\max}^2 l_a B_{aa} \log l_a} \right) \mtx{I_{l_a}} + \frac{\delta}{4} G_a \diag\left( \mtx{\theta_{(a)}} \right) + \frac{\delta}{10} \mtx{\theta_{(a)} \theta_{(a)}^T} \\
\succeq & \left(\frac{12m}{5} + \frac{\delta G_a}{4}\right) \diag\left(\mtx{\theta_{(a)}}\right) +\frac{\delta}{10} \mtx{\theta_{(a)} \theta_{(a)}^T} - \left(2\log l_a + \sqrt{6 \theta_{\max}^2 l_a B_{aa} \log l_a} \right) \mtx{I_{l_a}} \\
\succeq & \left(\frac{12m}{5} + \frac{\delta G_{\min}}{4}\right)\diag\left(\mtx{\theta_{(a)}}\right) - C \sqrt{\theta_{\max}^2 n p^+ \log n}  \mtx{I_{l_a}}.
\end{aligned}
\end{equation*}
Thus we have
\begin{equation}
\mtx{F_1} \succeq \left(\frac{12m}{5} + \frac{\delta G_{\min}}{4}\right)\diag\left(\mtx{\theta}\right) - C \sqrt{\theta_{\max}^2 n p^+ \log n}  \mtx{I_{n}}.
\end{equation}
Combining with the bound \eqref{eq:F2bound} to be proved later, we have 
\begin{equation*}
\mtx{\tilde{\Lambda_1}} = \mtx{F_1 + F_2} \succeq \left(\frac{12m}{5} + \frac{\delta G_{\min}}{4}\right)\diag\left(\mtx{\theta}\right) - C \sqrt{\theta_{\max}^2 n p^+ \log n}  \mtx{I_{n}} - C_0 \left( \sqrt{n \theta_{\max}^2 p^+ \log n} + \log n \right) \mtx{I_n}.
\end{equation*}

By taking
\begin{equation}
\delta \ge C \left( \frac{\sqrt{n \theta_{\max}^2 p^+ \log n}}{G_{\min}\theta_{\min}} + \frac{\log n}{G_{\min}\theta_{\min}} \right)
\end{equation}
when $ C $ is large enough, we obtain that 
\begin{equation}
\mtx{\tilde{\Lambda_1}} \succeq \frac{G_{\min}\delta}{8} \diag\left(\mtx{\theta}\right)
\end{equation}

With the above bound, to prove $ \mtx{\tilde{\Lambda}} \succ \mtx{0} $, it suffices to prove \begin{equation} \label{eqn:2.42}
\begin{bmatrix}
\frac{G_{\min}\delta}{8} \diag\left(\mtx{\theta}\right) & \mtx{\tilde{\Lambda_2}} \\
\mtx{\tilde{\Lambda_2^T}} & \mtx{\tilde{W} + \Xi}
\end{bmatrix} \succ \mtx{0}.
\end{equation}
Set $ w := \frac{C m}{G_{\min} \theta_{\min} \delta } $ with a sufficiently large constant $ C $. By multiplying  both sides of (\ref{eqn:2.42}) by 
$ \begin{bmatrix}
w\mtx{I_n} & \mtx{0} \\ \mtx{0} & \mtx{I_m}
\end{bmatrix} $, it suffices to prove \begin{equation}
\tilde{\tilde{\mtx{\Lambda}}} = \begin{bmatrix}
\frac{G_{\min}\delta w^2}{8} \diag\left(\mtx{\theta}\right) & w\mtx{\tilde{\Lambda_2}} \\ 
w \mtx{\tilde{\Lambda_2^T}} & \mtx{\tilde{W}+\Xi}
\end{bmatrix} \succ \mtx{0}.
\end{equation}
The above inequality is true if we can prove that the sum of absolute value of all off-diagonal entries is less than the absolute value of corresponding diagonal entry. 

For the first $ n $ rows, we have
\begin{equation}
w\mtx{\tilde{\Lambda_2}} = w \begin{bmatrix}
\lambda \mtx{d_{(1)} d_{(r+1)}^T}- \mtx{Z_1} - \frac{1}{G_1} \mtx{\theta_{(1)} \beta_1^T} \\
\vdots \\
\lambda \mtx{d_{(r)} d_{(r+1)}^T}- \mtx{Z_r} - \frac{1}{G_r} \mtx{\theta_{(r)} \beta_r^T}
\end{bmatrix}.
\end{equation}
Since $ \mtx{0} \le \mtx{Z_a} \le \mtx{ J_{(l_a, m)} } $, and by Lemma \ref{lemma:1}, $ \mtx{0} \le \mtx{\beta_a} \le \lambda \left(\tilde{d}_a + \tilde{d}_{r+1}\right) \mtx{d_{(r+1)}}$, the sum of absolute value of $ i $-th row of $ w\mtx{\tilde{\Lambda_2}} $ is no larger than \begin{equation}
\begin{aligned}
&w\left( \lambda d_i \tilde{d}_{r+1} + m + \frac{\theta_i}{G_a} \lambda \left(\tilde{d}_a + \tilde{d}_{r+1}\right)\tilde{d}_{r+1} \right).
\end{aligned}
\end{equation}
Therefore, we only need to prove \begin{equation} \label{eqn: 250}
\begin{aligned}
&\frac{G_{\min}\delta w^2}{8} \theta_i \ge w\left( \lambda d_i \tilde{d}_{r+1} + m + \frac{\theta_i}{G_a} \lambda \left(\tilde{d}_a + \tilde{d}_{r+1}\right)\tilde{d}_{r+1} \right) \\
\iff & C \frac{\theta_i}{\theta_{\min}}m \ge  \lambda d_i \tilde{d}_{r+1} + m + \frac{\theta_i}{G_a} \lambda\tilde{d}_a \tilde{d}_{r+1} + \frac{\theta_i}{G_a} \lambda\tilde{d}_{r+1}^2
\end{aligned}
\end{equation}
Note that $ H_a =\sum\limits_{b=1}^{r}B_{ab}G_b \ge n \bar{\theta} q^-$. The inequality 
(\ref{eqn: 250}) is implied by the following four conditions:
\begin{align}
C \frac{\theta_i}{\theta_{\min}}m \ge \lambda d_i \tilde{d}_{r+1} &\Leftarrow C \frac{\theta_i}{\theta_{\min}}m \ge \frac{2}{H_a} \left(B_{aa}-\frac{4}{5}\delta\right)\theta_i m(m+n) \\
&\Leftarrow C q^- n \ge B_{aa}(m+n) \Leftarrow C \ge \frac{B_{aa}}{q^-} \left(1 + \frac{m}{n}\right),
\end{align}
and
\begin{align}
C \frac{\theta_i}{\theta_{\min}}m \ge m,
\end{align}
and
\begin{align}
C \frac{\theta_i}{\theta_{\min}}m \ge \frac{\theta_i}{G_a} \lambda\tilde{d}_a \tilde{d}_{r+1} &\Leftarrow C \frac{\theta_i}{\theta_{\min}}m \ge\frac{\theta_i}{G_a} \frac{2}{H_a}\left(B_{aa} - \frac{4}{5}\delta\right) G_a m(m+n)\\
&\Leftarrow C q^- n \ge B_{aa}(m+n) \Leftarrow C \ge \frac{B_{aa}}{q^-} \left(1 + \frac{m}{n}\right),
\end{align}
and
\begin{align}
C \frac{\theta_i}{\theta_{\min}}m \ge \frac{\theta_i}{G_a} \lambda\tilde{d}_{r+1}^2 &\Leftarrow C G_a H_a^2 \ge \theta_{\min}(B_{aa}- \delta)m(m+n)^2 \\
&\Leftarrow C \frac{\theta_{\min}m}{q^-} {q^-}^2 n^2 \bar{\theta}^2 \ge \theta_{\min} B_{aa} m(m+n)^2\\
&\Leftarrow C \ge \frac{1}{\bar{\theta}^2}\frac{B_{aa}}{q^-} \left(1+\frac{m}{n}\right)^2.
\end{align}
The last inequality is due to the fact that $ G_a \ge \theta_{\min} l_{\min} \ge \frac{\theta_{\min}m}{q^-} $ and $ H_a \ge q^- n \bar{\theta}_{\min} $ where $ \bar{\theta} $ is the average value of all $ \theta_i $'s.


To study the bottom $ m $ rows of $ \mtx{\tilde{\tilde{\mtx{\Lambda}}}} $, we notice that \begin{equation} \label{eqn:255}
w\mtx{\tilde{\Lambda_2}^T} = \left[ \lambda \mtx{d_{(r+1)} d_{(1)}^T}- \mtx{Z_1^T} - \frac{1}{G_1} \mtx{\beta_1 \theta_{(1)}^T}, \cdots,  \lambda \mtx{d_{(r+1)} d_{(r)}^T}- \mtx{Z_r^T} - \frac{1}{G_r} \mtx{\beta_r \theta_{(r)}^T} \right],
\end{equation}
so the sum of all absolute values of $ j $-th row  of $ w\mtx{\tilde{\Lambda_2}^T} $ is not larger than 
\begin{equation}
\begin{aligned}
& w\left(\lambda d_{n+j} \sum\limits_{i=1}^{n} d_i + \sum\limits_{a=1}^{r} \left( \mtx{e_j^T Z_a^T 1_{l_a}} + \beta_{a_j} \right)\right) \\
\le & w \left(\lambda d_{n+j} \sum\limits_{i=1}^{n} d_i + \sum\limits_{a=1}^{r} \left(\alpha d_{n+j} x_{a_j} + \lambda d_{n+j} \sum\limits_{s=1}^{m} d_{n+s}x_{a_s}  + \xi_j x_{a_j} + \lambda \tilde{d}_a d_{n+j} \right) \right) \\
\le & w \left( 2\lambda d_{n+j} \sum\limits_{i=1}^{n} d_i + \alpha d_{n+j}\sum\limits_{a=1}^{r}x_{a_j} + \lambda d_{n+j} \left[\sum\limits_{s=1}^{m} d_{n+s} \left(\sum\limits_{a=1}^{r} x_{a_s}\right)\right] + \xi_j \sum\limits_{a=1}^{r}x_{a_j}  \right) \\
\le & w\left( 2\lambda d_{n+j} \sum\limits_{i=1}^{n} d_i + \left(\alpha d_{n+j} + \xi_j \right)\sqrt{r} + \lambda \sqrt{r} d_{n+j} \tilde{d}_{r+1} \right)
\end{aligned}
\end{equation}
On the other hand, the sum of absolute values of off-diagonal entries in the $ j $-th row of $ \mtx{\tilde{W} + \Xi} = \alpha \diag\left(\mtx{d^{*}_{(r+1)}}\right) + \lambda \mtx{d_{(r+1)}}\mtx{d_{(r+1)}^T} - \mtx{W} + \mtx{\Xi} $ is no larger than
\begin{equation}
\begin{aligned}
\lambda d_{n+j} \left(\tilde{d}_{r+1} - d_{n+j}\right) + (m-1).
\end{aligned}
\end{equation}
The $ j $-th diagonal entry of $ \mtx{\tilde{W} + \Xi} $ is no smaller than \begin{equation}
\alpha d^{*}_{n+j} + \lambda d_{n+j}^2 + \xi_j.
\end{equation}
Combining pieces, we see that it suffies to establish
\begin{equation}
\begin{aligned}
&\alpha d^{*}_{n+j} + \lambda d_{n+j}^2 + \xi_j > \\&\lambda d_{n+j} \left(\tilde{d}_{r+1} - d_{n+j}\right) + (m-1) + \\
& w\left( 2\lambda d_{n+j} \sum\limits_{i=1}^{n} d_i + \left(\alpha d_{n+j} + \xi_j \right)\sqrt{r} + \lambda \sqrt{r} d_{n+j} \tilde{d}_{r+1} \right) \\
\Leftarrow & \left(1 - w \sqrt{r}\right) \left(\alpha  d^{*}_{n+j} + \xi_j\right) + 2\lambda d_{n+j}^2 > (1+w\sqrt{r}) \lambda d_{n+j}\tilde{d}_{r+1} + (m-1) + 2w \lambda d_{n+j}\sum\limits_{i=1}^{n}d_i .\\
\end{aligned}
\end{equation}
By requiring $ \delta \ge C \frac{m \sqrt{r}}{\theta_{\min}G_{\min}} $, we have $ w\sqrt{r} = \frac{m \sqrt{r}}{\theta_{\min}G_{\min}\delta} < \frac{1}{2} $ for sufficiently large constant $ C $. Thus we only need to prove
\begin{equation} \label{eqn: 260}
\alpha  d^{*}_{n+j} + \xi_j + 4\lambda d_{n+j}^2 > 3 \lambda d_{n+j}\tilde{d}_{r+1} + 2(m-1) + 4w \lambda d_{n+j}\sum\limits_{i=1}^{n}d_i
\end{equation}
Notice that $ d^{*}_{n+j} \ge \max\left\{ d_{n+j}, \max\limits_{a}H_a \right\}  $, so  the equality~\eqref{eqn: 260} is implied by the conditions:
\begin{align*}
\alpha d^{*}_{n+j} \ge 3 \lambda d_{n+j} \tilde{d}_{r+1} \Leftarrow C\frac{m}{H_a} \ge C^{'} \frac{B_{aa}}{H_a^2} m(m+n) \Leftarrow C H_a \ge B_{aa}(m+n) \\
\Leftarrow C \ge \frac{1}{\bar{\theta}} \frac{B_{aa}}{q^-}\left(1+\frac{m}{n}\right),
\end{align*}
and
\begin{align*}
\alpha d^{*}_{n+j} \ge 2m \text{ holds since }  \alpha \ge C \text{max} \frac{m}{H_a} \text{ and } d_{n+j} \ge H^+,
\end{align*}
and
\begin{align*}
\alpha d^{*}_{n+j} \ge 4w\lambda d_{n+j} \sum\limits_{i=1}^{n}d_i \Leftarrow \alpha \ge C\frac{m}{G_{\min}\theta_{\min}\delta} \lambda \sum\limits_{i=1}^{n}d_i \Leftarrow \delta \ge \frac{Cm}{\alpha G_{\min}\theta_{\min}} \sum\limits_{i=1, i\in C_a^*}^{n} \frac{\theta_i}{H_a}\left(B_{aa}-\frac{4}{5}\delta\right)\\
\Leftarrow \delta \ge \frac{Cm}{\alpha G_{\min}\theta_{\min}} \sum\limits_{a}\frac{G_a B_{aa}}{H_a} \Leftarrow \delta \ge \frac{Cm}{\alpha G_{\min}\theta_{\min}} \frac{p^+}{q^-} \Leftarrow \delta \ge \frac{Cm}{\alpha G_{\min}\theta_{\min}}.
\end{align*}
Note that the condition~(\ref{eqn: delta}) in Theorem~\ref{thm:exact} fulfills all the requirements above. We conclude that $ \mtx{\Lambda} \succ \mtx{0} $.

Finally, we have 
\begin{equation}
\begin{aligned}
S_4 &= \langle \mtx{X}^*-\mtx{X}, \mtx{\Lambda}\rangle\\
&= \langle \mtx{V}^*\mtx{V}^{*^T}, \mtx{\Lambda}\rangle - \langle \mtx{X}, \mtx{\Lambda}\rangle
\le 0,
\end{aligned}
\end{equation}
where the last inequality is due to the fact that $ \mtx{\Lambda V^{*}} = \mtx{0} $ and $\mtx{X}$ and $\mtx{\Lambda}$ are both positive semi-definite matrix.

\subsection{Concluding the proof}
In conclusion, we have proved that $S_1<0$ and $S_2, S_3, S_4 \le 0$. Thus $ \Delta(\mtx{X}) = \langle \mtx{X}^*-\mtx{X}, \mtx{E}\rangle = S_1 + S_2 + S_3 + S_4 < 0 $ and we have finished the proof of Theorem~\ref{thm:exact}.

\subsection{Technical Lemmas} 

\begin{lemma}[Chernoff's Inequality] Let $ X_1, X_2, \cdots, X_n $ be independent random variables with \[
	\mathbb{P}(X_i=1) = p_i, \quad \quad \mathbb{P}(X_i=0) = 1-p_i.
	\]
	Then the sum $ X = \sum_{i=1}^{n} X_i $ has expectation $ \mathbb{E}(X) = \sum_{i=1}^{n}p_i $ and we have 
	\begin{equation*}
	\mathbb{P}\left(X \le \mathbb{E}(X) - t \right) \le e^{\frac{-t^2}{2\mathbb{E}(X)}}
	\end{equation*}
	and \begin{equation*}
	\mathbb{P}\left(X \ge \mathbb{E}(X) + t \right) \le e^{\frac{-t^2}{2\mathbb{E}(X)+t/3}}.
	\end{equation*}
\end{lemma}

	\begin{lemma} \label{lemma2}
	If we define $ f_i = \expect{d_i} $, then with probability at least $ 1-\frac{1}{n^2} $, we have for all $ i = 1, 2, \cdots, n $,
	\begin{align}
		 d_i - f_i &\le 2 \log n + \sqrt{6 f_i \log n}, \label{eqn:2.4} \\
		 d_i - f_i &\ge -\sqrt{6 f_i \log n} \label{eqn:2.5}.
	\end{align}
	Further, if we assume {\color{black} $ \delta \ge C_1 \sqrt{\frac{p^+ \log n}{\theta_{\min} G_{\min}}} $}, where $ p^+ = \max\limits_{a}B_{aa} $, we have \begin{equation} \label{eq:lemma2}
		\left| d_i - f_i \right| \le \frac{f_i}{5p^+}\delta.
	\end{equation} 
	\begin{proof}
		The inequalities (\ref{eqn:2.4}) and (\ref{eqn:2.5}) are the straightforward consequences of Chernoff's Inequality. These inequalities imply that $ \left| d_i - f_i \right| \le 2\log n + \sqrt{6f_i \log n} $. Since $ f_i = \theta_i H_a \ge q^- \theta_{\min} G_{\min} $, it follows from the assumption of Theorem~\ref{thm:exact} that
		\begin{equation}
			\delta \ge C \sqrt{\frac{p^+ \log n}{\theta_{\min} G_{\min}}} \ge C \sqrt{\frac{q^- \log n}{\theta_{\min} G_{\min}}} = C q^-\sqrt{\frac{\log n}{q^- \theta_{\min} G_{\min}}} \ge C q^-\sqrt{\frac{\log n}{f_i}}.
		\end{equation}
		Therefore, as long as $ C $ is large enough, we have $ \sqrt{\log n / f_i} \le 1$. Thus $ \frac{\log n}{f_i} \le \sqrt{\frac{\log n}{f_i}} \le 1 $ and  (\ref{eq:lemma2}) follows immediately.
	\end{proof}
\end{lemma}

%
%

\begin{lemma} \label{lemma3}
With high probability at least $ 1-\frac{2}{n}-\frac{2r}{n^2} $, the following inequalities hold for all $ 1 \le a < b \le r $:
	\begin{eqnarray}
	& \mtx{1_{l_a}^T K_{ab} 1_{l_b} } \ge G_a G_b B_{ab} - \sqrt{6 G_a G_b B_{ab} \log n} \ge G_a G_b \left(B_{ab}-\frac{1}{25}\delta\right)\label{eq:Kabsum} \\
	& \mtx{K_{ab}1_{l_b}} \le G_b \left(B_{ab}+\frac{1}{20}\delta\right) \mtx{\theta_{(a)}}\label{eq:Kabrow}\\
	& \mtx{1_{l_a}^TK_{ab}} \le G_a \left(B_{ab}+\frac{1}{20}\delta\right) \mtx{\theta_{(b)}^T}\label{eq:Kabcol} \\
	&\lambda \mtx{d_{(a)} d_{(b)}^T} \ge \left(B_{ab}+\frac{6}{25}\delta \right) \mtx{\theta_{(a)} \theta_{(b)}^T}. \label{eq:dadb}
	\end{eqnarray}
\end{lemma}

	\begin{proof}
		\textbf{Proof of (\ref{eq:Kabsum}):}
		The entry on the $ i $-th row of and $ j $-th column $ \mtx{K} $ follows the Bernoulli distribution of mean $ \theta_i \theta_j B_{ab} $. Thus the sum of all entries of $ \mtx{K_{ab}} $ has a mean of $ G_a G_b B_{ab} $. By Chernoff's Inequality, we have 
		\begin{equation}
		\mathbb{P}\left( \mtx{1_{l_a}^T} \mtx{K} \mtx{1_{l_b}} \le G_a G_b B_{ab} - t \right) \ge e^{\frac{-t^2}{2G_a G_b B_{ab}}}.
		\end{equation}
		Let $ t = \sqrt{6G_a G_b B_{ab}\log n} $, we have with probability at least $ 1-\frac{1}{n^3} $,\[
		\mtx{1_{l_a}^T} \mtx{K} \mtx{1_{l_b}} \ge G_a G_b B_{ab} - \sqrt{6G_a G_b B_{ab}\log n}.
		\]
		Note that $ \delta \ge C  \sqrt{\frac{p^+ \log n}{\theta_{\min} G_{\min}}} \ge C  \sqrt{\frac{B_{ab} \log n}{G_a G_b}} $ hold for sufficiently large constant $ C $, we have $ \mtx{1_{l_a}^T} \mtx{K} \mtx{1_{l_b}} \ge \left(B_{ab}-\frac{1}{25}\delta\right) $, and therefore (\ref{eq:Kabsum}) holds.
		
		\textbf{Proof of (\ref{eq:Kabrow}) and (\ref{eq:Kabcol}):}
		The entry on the $ i $-th row of and $ j $-th column $ \mtx{K} $ follows the Bernoulli distribution of mean $ \theta_i \theta_j B_{ab} $. Thus the sum of $ i $-th row of $ \mtx{K_{ab}} $ has a mean of $ \theta_i G_b B_{ab} $. By Chernoff's Inequality, we have 
		\begin{equation*}
			\mathbb{P}\left(\sum\limits_{j \in C_b^*} K_{ij} \ge \theta_i G_b B_{ab}+t \right) \ge e^{\frac{-t^2}{2(\theta_i G_b B_{ab}+t/3)}}.
		\end{equation*}
		Let $ t = 2\log n + \sqrt{6\theta_i G_b B_{ab}\log n} $, we have with probability at least $ 1-\frac{1}{n^3} $,
		\[
			\sum\limits_{j \in C_b^*} K_{ij} \le \theta_i G_b B_{ab} + 2\log n + \sqrt{6\theta_i G_b B_{ab}\log n}.
		\]
		Note that $ \delta > C \frac{\log n}{\theta_{\min} G_{\min}} \ge C \frac{\log n}{\theta_{i} G_{b}} $ and $ \delta \ge C  \sqrt{\frac{p^+ \log n}{\theta_{\min} G_{\min}}} \ge C  \sqrt{\frac{B_{ab} \log n}{\theta_{i} G_{b}}} $ hold for sufficiently large constant $ C $. It follows that $ \sum\limits_{j \in C_b^*} K_{ij} \le \theta_i G_b(B_{ab}+\frac{1}{20}\delta) $, and therefore (\ref{eq:Kabrow}) holds. The bound (\ref{eq:Kabcol}) can be proved similarly.
		
		\textbf{Proof of (\ref{eq:dadb}):}
		By Lemma \ref{lemma2}, for $ i \in C_a^* $ and $ j \in C_b^* $, we have $ d_i \ge \left(1-\frac{\delta}{5p^+}\right) \theta_i H_a $ and $ d_j \ge \left(1-\frac{\delta}{5p^+}\right) \theta_j H_b $. Note that \begin{equation*}
			\begin{aligned}
			\lambda \left(1-\frac{\delta}{5p^+}\right)^2 H_a H_b &\ge \frac{B_{ab}+\delta}{H_a H_b} \left(1-\frac{\delta}{5p^+}\right)^2 H_a H_b \ge (B_{ab}+\delta) \left(1-\frac{\delta}{5B_{ab}}\right)^2\\
			&=B_{ab} +\frac{3}{5}\delta -\frac{6\delta^2}{25B_{ab}} +\frac{4\delta^3}{25B_{ab}^2} \ge B_{ab} + \frac{6}{25}\delta.
			\end{aligned}
		\end{equation*}
	It follows that $ \lambda \mtx{d_{(a)}} \mtx{d_{(a)}}^T \ge \lambda\left(1-\frac{\delta}{5p^+}\right)^2 H_a H_b \mtx{\theta_{(a)}} \mtx{\theta_{(b)}}^T \ge \left(B_{ab} + \frac{6}{25}\delta\right) \mtx{\theta_{(a)}} \mtx{\theta_{(b)}}^T$, which finishes the proof.
\end{proof}

\begin{lemma}[{\citealp[Lemma 5]{CLX2017}}]\label{lem:spectral}
	Let $ \mtx{A} = \left\{a_{ij}\right\}_{n\times n}$ be a symmetric random matrix. Moreover, suppose that $ a_{ij} $ are independent zero-mean random variables satisfying $ |a_{ij}| \le 1 $ and $ \var(a_{ij}) \le \sigma^2 $. Then with probability at least $ 1-\frac{c}{n^4} $, we have \begin{equation*}
		\| \mtx{A} \| \le C_0 \left(\sigma \sqrt{n\log n} + \log n \right)
	\end{equation*}
	for some numerical constant $ c $ and $ C_0 $.
\end{lemma}

\begin{lemma}
	With high probability at least $ 1 - c\frac{r}{l_{\min}^4} $, we have 
	\begin{equation} \label{eqn: norm}
	\norm{B_{aa} \mtx{\theta_{(a)} \theta_{(a)}^T} -\mtx{K_{aa}}} \le 2\log l_a + \sqrt{6 \theta_{\max}^2 l_a B_{aa} \log l_a} 
	\end{equation}
	and 
	\begin{equation} \label{eq:F2bound}
	\norm{\mtx{F_2}} \le C_0 \left( \sqrt{n \theta_{\max}^2 p^+ \log n} + \log n \right).
	\end{equation}
\end{lemma}

\begin{proof}
	Note that the element $ B_{aa}\theta_i^2 - K_{ii} $ is a random variable with zero mean and variance of $ \theta_i^2 B_{aa} \left(1-\theta_i^2 B_{aa}\right) \le \theta_{\max}^2 B_{aa} $. Therefore,the matrix $ B_{aa} \mtx{\theta_{(a)} \theta_{(a)}^T} -\mtx{K_{aa}} $ satisfies the condition of Lemma~\ref{lem:spectral} with $ \sigma = \sqrt{\theta_{\max}^2 B_{aa}} $. Thus, with probability at least $ 1-\sum_{a=1}^{r} \frac{c}{l_a^4}$, we have \begin{equation}
		\norm{B_{aa} \mtx{\theta_{(a)} \theta_{(a)}^T} -\mtx{K_{aa}}} \le C_0 \left(\sqrt{\theta_{\max}^2 B_{aa} l_a \log l_a} + \log l_a\right)
	\end{equation}
	for some numerical constant $ c$ and $ C_0 $. 
	
	By a similar argument, we can prove that (\ref{eq:F2bound}) holds.
\end{proof}

\appendixpage
\appendix

\section{Proof of Lemma~\ref{lemma:1}}
\label{sec:proof_lemma:1}

	Consider the $ j $-th row of the matrix $ \alpha \diag\left( \mtx{d}_{(r+1)}^* \right) - \mtx{W} $. The sum of absolute values of the diagonal entries is at most $ m-1 $, whereas the absolute value of the corresponding diagonal entry is at least $ \alpha d_{n+j}^* - 1 $. Notice that $ \alpha d_{n+j}^* \ge \left(C\max\limits_{1\le a \le r} \frac{m}{H_a}\right)H^+ \ge Cm$. Therefore, for sufficiently large constant $ C $ (actually we only require $ C > 1 $), we can prove that the diagonal entry is larger than the sum of absolute values of the diagonal entries. Gershgorin Theorem \citep{horn2012} states that a matrix $ \mtx{A} = {a_{ij}}_{n\times n} $ is a positive definite matrix if  $ |a_{ii}| > \sum_{j\neq i}a_{ij} $ for all $ i=1,2,\cdots, n $. Therefore, we obtain that the matrix
	$ \alpha \diag\left( \mtx{d}_{(r+1)}^* \right) - \mtx{W} $ is a positive definite matrix. On the other hand, it is clear that $ \lambda \mtx{d}_{(r+1)} \mtx{d}_{(r+1)}^T $ is a positive definite matrix. We conclude that the matrix $ \mtx{\tilde{W}} := \alpha \diag\left(\mtx{d^{*}_{(r+1)}}\right) + \lambda \mtx{d_{(r+1)}}\mtx{d_{(r+1)}^T} - \mtx{W}$ is a positive definite matrix. This implies that the objective function of the optimization problem (\ref{eq:dual}) is strongly convex. The feasible set of the constraint (\ref{eq:dual}) is convex and compact, so the optimal solution exists uniquely.
	
    It is easy to see that there exist feasible solutions to the optimization problem (\ref{eq:dual}) with all inequalities satisfied strictly. Therefore, by the constraint qualification under the Slater’s condition, we know that the solution $ \vct{x_1}\cdots, \vct{x_r}$ must satisfy the KKT condition in (\ref{eq:KKT1}), (\ref{eq:KKT2}), and (\ref{eq:KKT3}).
    
    Since $ \mtx{\tilde{W} x_a} + \mtx{\tilde{Z}_a^T 1_{l_a}} = \mtx{\beta_a - \Xi x_a} $ and $ \langle \mtx{x}_a, \mtx{\beta}_a \rangle = 0 $, we have \begin{equation}
	\mtx{x}_a^T \left( \mtx{\tilde{W}} + \mtx{\Xi} \right) \mtx{x}_a = - \mtx{x}_a^T \mtx{\tilde{Z}_a^T 1_{l_a}} \le \mtx{1}_{l_a}^T \mtx{Z}_a^T \mtx{1_{l_a}} \le ml_a.
	\end{equation}
	Because $ \mtx{\Xi} $ is a non-negative diagonal matrix, $ \mtx{\tilde{W} + \Xi} $ is positive definite. By Cauchy-Schwarz Inequality, for all $ 1 \le a, b \le r $, we have \begin{equation}
	\mtx{x}_a^T \left( \mtx{\tilde{W} + \Xi} \right)\mtx{x}_b \le \left[\mtx{x}_a^T \left( \mtx{\tilde{W} + \Xi} \right)\mtx{x}_a\right]^{\frac{1}{2}}\left[\mtx{x}_b^T \left( \mtx{\tilde{W} + \Xi} \right)\mtx{x}_b\right]^{\frac{1}{2}} \le m\sqrt{l_a l_b}.
	\end{equation}
	
	Notice that equation (\ref{eq:KKT1}) is equivalent to \begin{equation}
	\left(\alpha \diag\left(\mtx{d^{*}_{(r+1)}}\right) + \lambda \mtx{d_{(r+1)}}\mtx{d_{(r+1)}^T} - \mtx{W} + \mtx{\Xi} \right) \mtx{x_a} = \mtx{\beta_a} - \lambda \tilde{d_a} \mtx{d_{(r+1)}} + \mtx{Z_a^T}\mtx{1_{l_a}}.
	\end{equation}
	Taking the $ j $-th row yields and using the non-negative property of $ \mtx{W} $ and $ \mtx{x_a} $, we have 
	\begin{equation}
	\begin{aligned}
	\alpha d^{*}_{n+j} x_{a_j} + \lambda d_{n+j} \sum\limits_{s=1}^{m} d_{n+s}x_{a_s}  + \xi_j x_{a_j} + \lambda \tilde{d}_a d_{n+j} &= \beta_{a_j} + \mtx{e_j^T Z_a^T 1_{l_a}} + \sum\limits_{s=1}^{m}W_{js}x_{a_s} \\
	&\ge \beta_{a_j} + \mtx{e_j^T Z_a^T 1_{l_a}}.
	\end{aligned}	
	\end{equation}
	
	Finally, since $ x_{a_j}\beta_{a_j} = 0 $, if $ \beta_{a_j} > 0 $ (thus $ x_{a_j} = 0 $), we have \begin{equation}
	0 \le \beta_{a_j} \le \lambda d_{n+j} \left(\tilde{d}_{r+1} - d_{n+j}\right) + \lambda\tilde{d}_a d_{n+j},
	\end{equation}
	or equivalently
	\begin{equation}
	\mtx{0} \le \mtx{\beta_a} \le \lambda \left(\tilde{d}_a + \tilde{d}_{r+1}\right) \mtx{d_{(r+1)}}.
	\end{equation}

\end{document}